\documentclass{article}
\usepackage{enumitem}
\usepackage[preprint]{neurips_2026}

\usepackage[utf8]{inputenc} % allow utf-8 input
\usepackage[T1]{fontenc}    % use 8-bit T1 fonts
\usepackage{hyperref}       % hyperlinks
\usepackage{booktabs}       % professional-quality tables
\usepackage{amsfonts}       % blackboard math symbols
\usepackage{nicefrac}       % compact symbols for 1/2, etc.
\usepackage{microtype}      % microtypography
\usepackage{xcolor}         % colors
\usepackage{amsmath,amssymb,amsthm,mathtools}
\usepackage{mathrsfs}
\usepackage{textcomp}
\usepackage{subcaption}
\usepackage{dsfont}
\usepackage{booktabs}
\usepackage{hyperref}

\newtheorem{theorem}{Theorem}

\newtheorem{proposition}[theorem]{Proposition}
\newtheorem{lemma}[theorem]{Lemma}
\newtheorem{corollary}[theorem]{Corollary}
\theoremstyle{remark}
\newtheorem{remark}[theorem]{Remark}
\title{Born Discrete, Made Smooth:\\ Variational Formulation of Shallow Neural Networks }
\author{  
Matej Benko \\
  Institute of Mathematics, Faculty of Mechanical Engineering, Brno University of Technology \\
  Technická $2896/2$, $616\,69$, Brno, Czech Republic \\
  \texttt{Matej.Benko@vutbr.cz}
  \And
Pierre Bousquet\\
Université de Toulouse, INSA Toulouse, CNRS, IMT, F-31062 Toulouse Cedex 9, France\\
\texttt{pierre.bousquet@math.univ-toulouse.fr} 
   \And
 Iwona Chlebicka \\
  University of Warsaw \\
  ul. Banacha 2, 02-097 Warsaw, Poland\\
   \texttt{i.chlebicka@mimuw.edu.pl} \\
   \AND
B{\l}a\.zej Miasojedow \\
  University of Warsaw \\
  ul. Banacha 2, 02-097 Warsaw,  Poland\\\texttt{b.miasojedow@mimuw.edu.pl}  
}

\usepackage{dutchcal}

\newcommand{\R}{{\mathbb{R}}}

\newcommand{\wt}{\widetilde}

\newcommand{\ve}{{\varepsilon}}

\newcommand{\cF}{{\mathscr{F}}}

\newcommand{\cJ}{{\mathscr{J}}}

\newcommand{\cR}{{\mathscr{R}}}
\newcommand{\cM}{{\mathscr{M}}}
\newcommand{\cMat}{{\cM^{\rm at}}}
\newcommand{\cW}{{\mathcal{W}}}
\newcommand{\dx}{{\,{\rm d}x}}

\renewcommand{\d}{{\,{\rm d}}}
\newcommand{\dth}{{\,{\rm d}\theta}}
\newcommand{\dmu}{{\,{\rm d}\mu}}

\newcommand{\dy}{{\,{\rm d}y}}
\newcommand{\vt}{{\vartheta}}

\newcommand{\co}{c_\omega}

\def\d{{\,{\rm d}}}

\newcommand{\supp}{{\rm supp\,}}

\definecolor{darkgreen}{rgb}{0.00, 0.50, 0.00} 
\definecolor{airforceblue}{rgb}{0.36,0.54, 0.66}

\begin{document}

\maketitle

 \begin{abstract}
%Although neural networks are remarkably effective, their underlying principles remain elusive. By replacing the training problem with a well-posed variational surrogate, we show that the resulting solution inherits remarkable properties that reveal what neural networks are implicitly doing. We identify a family of $\lambda$-convex functionals  over parameter densities in weighted Sobolev spaces. We prove that these variational problems are globally well-posed, stable, and rapidly convergent, exhibiting unexpected almost $C^3$ regularity.

%Unlike the Wasserstein approaches commonly used in training descriptions, which face limited regularity and discretization challenges, our approach focuses on the properties of the solutions. This gives us direct access to convex analysis and elliptic regularity while remaining consistent with finite-width networks via atomic measures. As a practical consequence, the optimal parameter density is obtained by solving a \emph{single linear system} (no iterative optimization required) with generalization error explicitly controlled by the regularization parameter $\alpha$ at rate $1/\alpha$, and finite-width networks of width $N$ achieving the continuum optimum within $O(1/N)$. This perspective complements existing NTK and transport-driven analyses by revealing a higher-order smoothness and stable dynamics previously difficult to obtain. Our results suggest new regularization strategies for infinite-width training and provide a pathway toward multilayer extensions.

Although neural networks are remarkably effective, their underlying optimization principles remain theoretically elusive, often characterized by non-convex landscapes and stochastic heuristics. In this work, we propose a paradigm shift by replacing the discrete training problem of shallow neural networks with a well-posed \textit{continuum variational surrogate}. We identify a family of $\lambda$-convex functionals over parameter densities in weighted Sobolev spaces and prove that these variational problems are globally well-posed, stable, and exhibit unexpected almost $C^3$ regularity. 

Unlike existing Wasserstein-based or Mean-Field approaches, which often face limited regularity and discretization challenges, our formulation provides direct access to \textit{elliptic regularity and convex analysis}. This allows us to prove that the optimal parameter density can be obtained by solving a \textit{single linear system}, bypassing iterative optimization entirely. We establish explicit generalization error controls at a rate of $1/\alpha$ relative to the regularization parameter, and prove that finite-width networks of size $N$ achieve the continuum optimum at an $O(1/N)$ rate. This perspective bridges the gap between the Neural Tangent Kernel (NTK) and feature-learning regimes, providing a principled framework for understanding over-parameterization through the lens of variational calculus.
\end{abstract}

\section{Introduction}
%Understanding the surprising effectiveness of optimizing highly overparameterized neural networks remains a central question in deep learning theory. A powerful approach models training variationally, as minimizing a smooth convex loss functional over measures, with gradient descent acting on  a sparse linear combination of atoms from a large parameterized family. This represents functions as integrals over parameter distributions and focuses on the associated gradient flow. As width grows, parameters collapse into a compact form, revealing convergence, feature learning, and enabling rigorous guarantees. We introduce a new family of variational functionals for shallow networks that, in the limit, ensure not only existence, uniqueness, and stability of minimizers, but also a level of smoothness rare in mean-field problems.

Understanding the surprising empirical effectiveness of gradient-based optimization on 
highly overparameterized, nonconvex objectives remains one of the central 
challenges in deep learning theory. Explanation of non-overfitting in such problems is in a short supply. A widely studied approach 
focuses on representing the 
network output as an integral over a distribution of hidden units with
\begin{equation}
    \label{eq-def-fN-intro}
f_N(x):=\tfrac1N\textstyle\sum_{i=1}^N w_i h(\theta_i,x)\,, \quad \text{
where \((w_i,\theta_i)_{i=1}^N\) stand for the unknown parameters,}
\end{equation}
and studies the minimization of a risk functional over the space of measures, with gradient descent interpreted as a flow on that space. This 
mean-field perspective has revealed deep connections to interacting particle 
systems, optimal transport, and Wasserstein gradient flows 
\citep{mei2018mean, chizat2018global, FRF, sirignano2020mean, 
rotskoff2018trainability, E2018, E2020}, and has led to convergence 
guarantees, insights into feature learning, and PDE-based analyses of training 
dynamics. In parallel, the neural tangent kernel (NTK) framework 
\citep{jacot2018ntk} and related overparameterization results 
\citep{du2019gradient, allenzhu2019convergence} explain why gradient descent 
succeeds in a linearized infinite-width regime. More recent efforts extend 
mean-field PDE analyses to deeper, residual architectures 
\citep{chen2024resnet}, sharpen particle and Langevin approximations 
\citep{nitanda2024particle, mousavi2025multiindex}.\citet{dereich2024minimizers} 
study existence of minimizers in the finite-width landscape via a closure of 
the search space; \citet{gribonval2025shallow} establish nearly optimal 
approximation rates for shallow $\mathrm{ReLU}^k$ networks on Sobolev classes. Both works operate at finite networks, 
whereas we concentrate on the regularity of the continuum minimizer. Finally, \citet{oh2025sobolev} add a Sobolev penalty on the network output over the 
input domain $D$ to improve optimization conditioning. In contrast, our regularization acts on the parameter density over $\Omega$ to enforce variational well-posedness and regularity. This makes the two approaches complementary yet structurally distinct. 

%}and introduce 
%Sobolev training objectives to improve conditioning and convergence 
%\citep{oh2025sobolev}.

%Despite this progress, two principal frameworks dominate the field and each entails a characteristic trade-off. Mean-field approaches based on Wasserstein geometry  support a particle interpretation of gradient descent and describe the evolution of parameter distributions via continuity equations, but convexity holds only in the displacement sense and is delicate even then, minimizers enjoy at best Lipschitz regularity, and the passage from the continuum to finite-width models requires non-trivial approximation arguments. The NTK framework achieves a genuinely convex objective, but does so by linearizing the network around its initialization; the resulting guarantees are local in nature and do not extend to the nonlinear, feature-learning regime. These two perspectives leave partly open a different and equally fundamental set of questions about the \emph{variational problem itself}: Does the infinite-width objective admit a unique and stable minimizer? What regularity do optimal parameter distributions possess? Can one obtain global convexity guarantees for a nonlinear continuum model, without linearization? And what quantitative convergence properties follow from the intrinsic geometry of the continuum problem?

Despite this progress, two principal frameworks dominate the field, each with characteristic trade-offs. Mean-field/Wasserstein approaches support a particle interpretation of gradient descent but offer only displacement convexity and require non-trivial approximation to pass to finite-width models. The NTK framework achieves a truly convex objective but only by linearizing around initialization, limiting guarantees to the lazy-training regime. Both leave partly open a different set of questions about the variational problem itself: Does the infinite-width objective admit a unique, stable minimizer? What regularity do optimal parameter distributions possess? Can one obtain global convexity for a nonlinear continuum model without linearization?

In this paper we propose a third route that addresses these questions directly. 
We formulate shallow neural network training as a variational problem over 
\emph{parameter densities} in a weighted Sobolev space $\mathcal{W} = 
W^{1,2}(\Omega)\cap L^2_\omega(\Omega)$ forming a linear Hilbert-space setting 
that retains the infinite-width mean-field spirit while gaining direct access 
to convex analysis, elliptic PDE theory, and quantitative gradient flow 
estimates. The regularized objective on $\cW$ which is given by
\begin{equation*}%\label{Fab_intro}
\cF^{(f)}_{\alpha,\beta}(u)
=
\mathcal{R}(f, u\,\mathrm{d}\theta)
+ \alpha\|u\|_{L^2_\omega}^2
+ \beta\|\nabla u\|_{L^2}^2
\end{equation*}
combines a squared-risk term measuring data fidelity with $L^2_\omega$ and 
Sobolev regularization. The $L^2_\omega$ term induces coercivity and 
$\lambda$-convexity  for $\lambda>0$; the Sobolev term promotes smoothness and plays a role analogous to entropy regularization in the well-established, yet computationally hard, Wasserstein 
formulations \citep{JKO, ambrosio2008gradient, Santambrogio-survey}. In turn, the minimizer $u^*_f$ to $\cF^{(f)}_{\alpha,\beta}$ can be interpreted as the implicit target of NN in the overparameterized regime, while the regularization parameters $\alpha,\beta$ play the role of weight decay and smoothness bias.

Our Hilbert-space framework is built on the classical calculus of variations staying fully compatible with standard shallow network parametrizations. Although our analysis does not directly model stochastic gradient descent, the exponential convergence rate $e^{-\lambda t}$ of the gradient flow associated to our variational problem and the structure of the Euler--Lagrange equation suggest that the $L^2_\omega$ geometry may serve as a useful proxy for understanding the implicit bias in overparameterized networks. In particular, the $\lambda$-convex 
structure implies that any  discretization of the gradient flow inherits stability consistent with the 
empirical robustness of SGD near overparameterized minima 
\citep{mei2018mean, chizat2018global}. In our study,  the minimization of $\cF_{\alpha,\beta}^{(f)}$ can be reduced to the optimization of a quadratic functional depending on the coordinates of the competitors in a chosen  basis of $L^2_\omega(\Omega)$. Since the projection of the  minimizer onto a finite dimensional subspace of $L^2_\omega(\Omega)$  can be obtained by merely solving a linear system of equations, the optimization is a fast and easy procedure, as it does not require any iteration  or gradient descent arguments.

The variational formulation we propose yields a cluster of remarkable properties.

\textbf{No Lavrentiev gap.}
A key consistency result shows that the infimum of the risk is the same 
whether one optimizes over atomic measures (finite-width networks), Sobolev 
densities, or smooth compactly supported functions. No approximation error is 
introduced by passing to the continuum, and finite-width networks of width $N$ 
achieve the continuum optimum up to an $O(1/N)$ error. This places the 
continuum theory as an exact relaxation of the 
finite discrete training problem. The absence of Lavrentiev's phenomenon is a typical key step in inferring regularity of minimizers \citep{badisu,fomami,bcdfm,mira}. 

\textbf{$\lambda$-convexity.}
The functional $\cF_{\alpha,\beta}^{(f)}$ is $\lambda$-convex on 
$L^2_\omega$ with $\lambda = 2\alpha\geq 0$, which gives the exponential convergence of the gradient flow, and it  is $2\min(\alpha,\beta)$-convex on \(\cW\), which leads to the existence and the uniqueness of the minimizer, as well as its stability with respect to the data. These are global properties of the model, not consequences of  linearization, equipping with strong tools, which are typically not available in the Wasserstein approach,~\citep{chizat2025convergencedriftdiffusionpdesarising}. The $L^2_\omega$-gradient flow of $\mathcal{F}^{(f)}_{\alpha,\beta}$, the continuous-time analogue of gradient descent, converges exponentially fast to $u^*$, cf. \citep{AmBrSe1}. A~time-discrete implicit Euler scheme provides an approximation with explicit error, combining the exponential rate with a $\sqrt{\tau}$ term, where \(\tau\) is the discretization step.

\textbf{Near-$C^3$ smoothness.}
The minimizer $u^*$ of $\cF_{\alpha,\beta}^{(f)}$ solves an explicit Euler--Lagrange equation. Moreover, it is $C^{2,s}$ for every $s<1$, which gives a major structural advantage: it prevents 
pathological concentrations of mass in parameter space  and opens the door to higher-order numerical schemes via classical PDE methods. In contrast,  in the Wasserstein approach, regularity of minimizers is inherited from the activation function, whereas here we get surprising smoothness even in the case of merely locally Lipschitz activation functions, see \citep[Proposition~3]{mei2018mean}. This suggests that the predictions of the neural network concentrate around graphs of almost $C^3$-functions, which illustrates the non-overfitting phenomenon, robustness of training, and suggests that the neural networks are attracted to the lower-dimensional manifolds, cf.~\citep{Belkin2019Reconciling,Arora2019Implicit,NEURIPS2020_f21e255f}. The stability results we infer provide a concrete form of implicit bias toward smooth solutions and explain why the variational formulation yields stable and well-generalizing models.

\textbf{Contributions.} 
%Our framework is the theory of infinite-width  shallow neural networks and our goal is to provide a  formulation of it that is entirely variational.  We identify a class of regularized training objectives replacing the study on the infinite-width learning process with a surrogate that, falling into the realm of the classical calculus of variations, is globally well-posed, stable, and rapidly convergent. This brings a novel insight into what neural networks are implicitly doing. Let us list the details.
%
Our framework provides an entirely variational formulation for infinite-width shallow networks. Replacing the discrete training process with a regularized surrogate, we identify a class of globally well-posed and stable objectives. This approach follows three steps: (i) transforming the $N$-neuron optimization into a continuum problem over weighted Sobolev spaces; (ii) exploiting the structure to ensure a unique and almost $C^3$ regular minimizer; (iii) bridging to finite-width networks via atomic measures with $O(1/N)$ consistency. This offers novel insight into the implicit bias of neural networks, as detailed below.

\begin{itemize}[nosep, leftmargin=*]
    \item \textbf{A variational theory of shallow network training.}
    We formulate shallow neural network training as a Sobolev-space variational problem (Section~\ref{sec:Variational}) and concentrate on the properties of solutions. We prove that this continuum formulation is \emph{exact}: there is no Lavrentiev gap, so atomic, Sobolev, and smooth minimization 
    classes coincide (Theorem~\ref{theo:no-Lavr}). 
    We show that this infinite-width problem is quantitatively faithful to finite networks via an $O(1/N)$ approximation bound 
    (Proposition~\ref{prop:inf-equivalence}).

    \item \textbf{Global convexity and well-posedness beyond existing frameworks.}
    We introduce a regularized functional $\mathcal{F}_{\alpha,\beta}^{(f)}$ that is globally $2\min(\alpha, \beta)$-convex on $\cW$
    yielding existence, uniqueness, and stability of minimizers. 
      We establish that the associated $L^2_\omega$ gradient flow converges exponentially fast to equilibrium at rate $e^{-2\alpha t}$
    (Corollary~\ref{cor:exp-convergence}). In particular, this provides an explicit analytical mechanism for non-overfitting in the infinite-width regime. 

    \item \textbf{Implications to ML.} The minimizers to the regularized functional admit an Euler--Lagrange characterization and exhibit near-$C^3$ regularity via elliptic PDE arguments 
    (Theorem~\ref{thm:convexity}), a level of smoothness not accessible in standard infinite-width analyses. Propositions~\ref{prop:stability} and~\ref{prop:stabillity_L^infty}, together with Corollary~\ref{cor:generalization}, show quantitative dependence of the minimizer under perturbations providing a link between regularization, stability, and generalization.   

\end{itemize} 

\textbf{Shallowness.} Despite the dominance of deep architectures, the single-hidden-layer setting already captures the core variational complexity of neural network training: a nonlinear, non-convex objective over an infinite-dimensional space. Resolving these obstructions in this setting isolates the essential structure and provides a foundation for multilayer extensions \citep{chen2024resnet}.

\textbf{Limitations.} The current formulation is for shallow (one-hidden-layer) networks. Moving from one to more layers complicates the structure of the optimization problem, which becomes strongly nonlinear, see Section~\ref{sec:numerics}. In fact, for two hidden layers the parameter density lives on $\Omega_1 \times \Omega_2$ and the risk functional becomes a composition, making $\cR$ possibly nonconvex even after Sobolev regularization. The Euler--Lagrange system becomes a coupled nonlinear PDE rather than a linear equation, so existence of smooth minimizers requires additional structural assumptions such as small initial condition or a perturbative regime around a known solution. The gradient flow no longer reduces to a linear semigroup and convergence is elusive.

\textbf{Summary.}
Our approach is related to the mean-field literature 
\citep{mei2018mean, chizat2018global, FRF, nitanda2024particle, 
mousavi2025multiindex}, but operates in a different analytical regime.   The connection to 
Sobolev-type objectives recently explored in \citep{oh2025sobolev} is 
suggestive, though the focus there is on conditioning of the optimization 
problem rather than on variational well-posedness or regularity of minimizers. 
Our results also complement convex and function-space formulations of 
infinite-width networks \citep{bach2017breaking, ongie2019function}, which 
concentrate on description of the learning process; we focus instead 
on the variational structure of the objective and the regularity of its 
minimizers. Finally, while 
the NTK framework \citep{jacot2018ntk} achieves convexity through 
linearization and is well-suited to the lazy-training regime 
\citep{du2019gradient, allenzhu2019convergence}, our $\lambda$-convexity holds 
for the same nonlinear problem that arises in the infinite-width limit of 
shallow network training, without any linearization assumption.

\section{Variational formulation}\label{sec:Variational}

We develop a continuum approach to shallow networks based on parameter distributions. 

\subsection{Exact variational characterization of neural network training}
Let $D\subset \R^d$ be the input (a bounded Borel set equipped with the Lebesgue measure), let $f\in L^2(D)$ and let $\Omega\subset \R^{d+1}$ be the parameter space of hidden units, assumed to be open. 
%We also assume that $(1+|\cdot|^2)f(\cdot)\in L^1(D)$, which always holds when $D$ is bounded. 
Let  $\theta=(\theta_0,\theta')\in\Omega$ parametrize the hidden layer with $\theta_0\in \R$ being interpreted as  the bias, $\sigma$ be an activation function that we assume to be continuous with the typical choices of ReLU or sigmoid functions. We also consider the function  $h: \Omega\times D\to \R$ given by \begin{equation}
    \label{eq:h-def}h(\theta,x):=\sigma(\theta_0+\theta'\cdot x)\,
\end{equation}
and we require that there exists \(C_h>0\) such that for every \((\theta,x)\in \Omega\times D\),
\begin{equation}
\label{eq:h-ass}
|h(\theta,x)|\leq C_h(1+|x|)(1+|\theta|)\,.
\end{equation}
This condition holds true in the typical choices of ReLU and sigmoid functions. In this framework training a shallow neural network consists in approximating $f$ with functions $f_N$ that  are given by~\eqref{eq-def-fN-intro}.

The learning process amounts to the minimization of $\int_D (f(x)-f_N(x))^2\dx$ over $(w_i,\theta_i)_{i=1}^N$. 
Each function \(f_N\) can be rewritten as
\begin{equation}
\label{eq:fN-muN}
f_N(x)=\int_\Omega h(\theta,x)\dmu_N(\theta)\quad\text{for }\ \mu_N = \frac{1}{N} \sum_{i=1}^N w_i \delta_{\theta_i}\in\cMat(\Omega)\,.
\end{equation}
Observe that $\mu_N$ is a finite atomic signed measure  (with finite number of atoms). Denote by \(\cM_a(\Omega)\) the set of those  finite signed Borel measures $m$ such that \(\int_{\Omega}|\theta|^a\d |m|(\theta)<\infty.\) Then, we generalize \eqref{eq:fN-muN} for $m\in\cM_1(\Omega)$ by introducing the function
\begin{equation}
\label{eq:fm} 
f_m(x)=\int_\Omega h(\theta,x)\d m(\theta)\,.
\end{equation}  
In view of~\eqref{eq:h-ass} and the fact that \(m\in \cM_1(\Omega)\), the function \(f_m\) is well-defined and belongs to \(L^{2}(D)\). We can thus consider the population risk functional
\begin{equation}
    \label{eq:risk}
\cR(f,m):=\int_D (f(x)-f_m(x))^2\dx\,.
\end{equation}
The training problem becomes the minimization of $\cR(f,m)$ over $m\in\cM_1(\Omega)$, as suggested in~\cite[Section 2.3]{FRF}. To obtain a Hilbert-space formulation, we restrict attention to measures that admit a density, writing $\d m(\theta)=u(\theta)\dth$ for some Sobolev function $u$. Since $\Omega$ is typically unbounded, we work in a weighted space that controls integrability, so that all quantities under consideration are well-posed. Let $\omega:\Omega\to[1,\infty)$ be a smooth weight such that
\begin{equation}
\label{eq:weight-assumptions}
\co:=\textstyle\int_\Omega  {(1+|\theta|)^4}\tfrac{1}{\omega(\theta)}\dth <\infty\,.
%\qquad
%\int_\Omega \frac{1}{\omega(\theta)}\dth <\infty.
\end{equation}
%We denote $\bco: =\int_{\Omega}\frac{1}{\omega(\theta)}\dth$
%and $\hco:=\int_{\Omega}\frac{|\theta|^2}{\omega(\theta)}\dth$.
Moreover, we define
\[
L^2_\omega(\Omega)
=
\left\{
u:\textstyle\int_\Omega u(\theta)^2\omega(\theta)\dth<\infty
\right\}
\qquad\text{and}\qquad
\cW := W^{1,2}(\Omega)\cap L^2_\omega(\Omega)\,.
\]
Condition~\eqref{eq:weight-assumptions} and the Schwarz inequality imply that for every \(v\in L^{2}_\omega(\Omega)\), the measure \(v\dx\) belongs to \(\cM_1(\Omega)\).
Our main result establishes a striking equivalence: neural network training admits an exact variational characterization. In particular, the continuum formulation is not merely an approximation, but an exact relaxation of the finite discrete problem, which reveals its intrinsic variational structure.
\begin{theorem}
    \label{theo:no-Lavr} Under the conditions of this section, for every \(f\in L^2(D)\), one has:
\begin{equation}
    \label{eq:R-no-Lavr}
\inf_{\mu \in \cMat(\Omega)} \cR(f,\mu)=
\inf_{m \in \cM_1 (\Omega)} \cR(f,m)
=
\inf_{v \in \cW} \cR(f,v\dx)=
\inf_{v \in C_c^\infty(\Omega)} \cR(f,v\dx)\,.
\end{equation}
\end{theorem}
From a variational perspective, this shows that no Lavrentiev gap occurs for \(\cR(f,\cdot)\) across  measures, Sobolev maps, and smooth functions. In particular, minimizing sequences can be transferred between these classes without loss of optimality, which is a fundamental structural property for regularity theory. This bridges the discrete and continuum regimes at a deeper level, ensuring that analytical and numerical approximations faithfully capture the true variational behavior of the problem, and places neural network training firmly within the scope of variational analysis.

Let us point out the rate of convergence of the risk with respect to the width of the neural network in the spirit of \citep{mei2018mean}. Observe that {$L^{2}_\omega(\Omega)$} is a natural space to consider $\cR$ and further regularity results will be provided for the problem on the right-hand side below. 
%We observe that by the Schwarz inequality and the condition \eqref{eq:weight-assumptions}, any function in \(\cW\) is the density of a measure in \(\cM_1(\Omega)\).
In the next proposition, we denote by \(\cM^{\rm at}_N(\Omega)\) the set of those purely atomic signed measures on \(\Omega\) that involve at most \(N\) atoms. Observe that such measures can be written as in \eqref{eq:fN-muN} for suitable coefficients \(w_i\).
\begin{proposition}
\label{prop:inf-equivalence}
Let $\cR$ be given by~\eqref{eq:risk}. Then, there exists  \(C>0\) depending only on \(D, h, \omega\) such that for every \(f\in L^{2}(D)\), for every \(N\geq 2\), 
\[
\inf_{\mu\in \cM_{N}^{\rm at}(\Omega)}\cR(f,\mu) \leq \inf_{v\in L^{2}_\omega(\Omega)} \left(\cR(f,v\dx)+ \tfrac{C}{N}\|v\|_{L^{2}_\omega (\Omega)}^2\right)\,.
\]
As a consequence, if there exists a bounded minimizing sequence for \(\cR(f,\cdot)\) in \(L^{2}_\omega(\Omega)\), then
\[
\inf_{\mu\in \cM_{N}^{\rm at}(\Omega)}\cR(f,\mu) \leq \inf_{v\in L^{2}_\omega(\Omega)}\cR(f,v\dx)+O\left(\tfrac{1}{N}\right) 
\,.
\]
\end{proposition}

\subsection{Towards justification of non-overfitting: stability of training}

By expanding the risk term and using the Fubini theorem, we obtain 
for \(f\in L^{2}(D)\) and \(m\in \cM_1(\Omega)\), that
\begin{equation}
\label{eq:quadratic-decomposition}
\cR(f,m)
=
\|f\|_{L^2(D)}^2
-2\int_\Omega Q(\theta)\d m(\theta)
+\int_\Omega\int_\Omega K(\theta,\vt)\d m(\theta)\d m(\vt)\,,
\end{equation}
where
\begin{equation}
\label{eq:bK}
Q(\theta):=\int_D f(x)h(\theta,x)\dx\qquad\text{and}\qquad
K(\theta,\vt):=\int_D h(\theta,x)h(\vt,x)\dx\,.
\end{equation}
In view of  \eqref{eq:h-ass} for every \(\theta,  \vt\in \Omega\), for  \(C_{h}':=C_{h}\|1+|x|\|_{L^{2}(D)}\),  one has:
\begin{equation}\label{def-CKCQ}
 |Q(\theta)|\leq C_{h}\,(1+|\theta|)\int_{D}(1+|x|)|f(x)|\dx\qquad\text{and} \qquad |K(\theta,\vt)|\leq (C_{h}')^2\,(1+|\theta|)\,(1+| \vt|)\,.
\end{equation}
The functional \(\cR(f,\cdot)\) is convex on \(\cM_1(\Omega)\) (see Lemma~\ref{lm-cRf-convex} below) but not necessarily \emph{strictly} convex. We endow the Hilbert space $\cW$ with the norm  $\|u\|_{\cW}^2 := \|u\|_{L^{2}_\omega(\Omega)}^2 + \|\nabla u\|_{L^{2}(\Omega)}^2$.
We are thus led to introduce for every \(\alpha, \beta\geq 0\) the natural {\it regularized objective}
\begin{equation}
\label{eq:regularized-functional}
\cF^{(f)}_{\alpha,\beta}(u)
:=
\begin{cases}
\cR(f,u \dx)
+\alpha\|u\|_{L^2_\omega}^2
+\beta\|\nabla u\|_{L^2}^2\,,
& u\in \mathcal W\,,\\
+\infty\,, & u\in L^{2}_\omega(\Omega)\setminus \cW\,.
\end{cases}
\end{equation}
 The objective combines a term measuring data fidelity with Sobolev and $L^2_\omega$ regularization {terms} that provide coercivity and enforce regularity in the parameter space. In particular, the Sobolev term $\|\nabla u\|_{L^2}^2$ induces a diffusive effect in the associated gradient flow, analogous to entropy regularization in Wasserstein formulations well-designed for the study of the learning process, cf.~\cite{JKO,ambrosio2008gradient,Santambrogio-survey}. We avoid it as, in contrast to the computationally challenging Wasserstein approach, the $L^2_\omega$ setting is naturally compatible with the study of the solutions to standard shallow neural network parametrizations. This enables the use of convex-analytic techniques together with classical regularity theory.

We notice that, when restricted to its domain \(\cW\), 
the regularized functional from~\eqref{eq:regularized-functional} is Fréchet differentiable and $2\min(\alpha, \beta)$-convex, i.e., the functional \(u\mapsto\cF^{(f)}_{\alpha,\beta}(u)-\min(\alpha, \beta)\|u\|_{\cW}^2\) is convex on $\cW$. Moreover, on \(L^{2}_\omega(\Omega)\), the functional $\cF^{(f)}_{\alpha,\beta}$ is  $2\alpha$-convex.  

Some of our results require to assume that \(h\) is locally Lipschitz with respect to the first variable, in the  sense that for every \(R>0\), there exists \(c_{h,R}>0\) such that 
\begin{equation}\label{eq-Phi-loc-Lip}
|h(\theta, x) - h(\vt,x)|\leq c_{h,R}|\theta-\vt|(1+|x|)\quad \text{for every \(\theta, \vt\in \Omega\cap B_R\), for every \(x\in D\)}\,.
\end{equation}
The regularized functional enjoys  strong properties due to the convex structure of $  \cF^{(f)}_{\alpha,\beta}$. It ensures global existence and uniqueness of the minimizer $u_f^*$  together with an explicit Euler--Lagrange PDE. The almost $C^{3}$ regularity of $u_f^*$ is a major advantage: it lifts the optimal density far beyond generic solutions, opening the way to prove sharp asymptotic analysis, stability estimates, and a justification of particle approximation tools that are essential yet scarce in mean-field neural network theory. 

%\dg{For \(s \in (0,1)\), we say that
%\(u \in C^{2,s}_{\mathrm{loc}}(\Omega)\) if \(u \in C^2(\Omega)\) and its second derivatives are locally \(s\)-Hölder continuous, i.e. $[D^2u]_{C^{0,s}(K)}:= \sup_{\substack{x,y \in K\\x\neq y}}\frac{|D^2u(x)-D^2u(y)|}{|x-y|^s} <\infty$, for every compact $K \Subset \Omega$.}
\begin{theorem}
\label{thm:convexity} Under the conditions of this section, let $\alpha,\beta>0$ and $h$ be a continuous function satisfying~\eqref{eq:h-ass} and \eqref{eq-Phi-loc-Lip}. Then $\cF^{(f)}_{\alpha,\beta}$ is $2\min(\alpha, \beta)$-convex on \(\cW\) and $2\alpha$-convex on $L^{2}_\omega(\Omega)$. Moreover, it admits a unique minimizer $u^*\in\mathcal W$,
which is a classical solution to the Euler--Lagrange equation: 
\begin{equation}
\label{eq:euler-lagrange}
-\beta \Delta u^*(\theta)
+ \alpha \omega(\theta) u^*(\theta)
-Q(\theta)+\int_{\Omega}K(\theta,\vt)u^*(\vt)\d\vt
= 0\,,
\end{equation} 
and it holds that $u^* \in C^{2,s}_{\mathrm{loc}}(\Omega)$  for every  $s\in(0,1)$. 
\end{theorem}

We stress that due to convexity, our variational problem is stable with respect to the target. %In the results below, we stress the dependence of the functional \(\cF_{\alpha, \beta}\) on \(f\) and indicate it explicitly with the notation \(\cF^{(f)}_{\alpha, \beta}\):
\begin{proposition}\label{prop:stability} Under the conditions of this section, let $\alpha, \beta>0$ and $h$ be a continuous function satisfying~\eqref{eq:h-ass}. Then, the unique minimizer  \(u_f^*\in \cW\)  of \(\cF_{\alpha, \beta}^{(f)}\) depends linearly on \(f\) and
there exists \(C=C(\Omega,D,\omega,h)>0\) such that
\[
\|u_f^*\|_{L^{2}_\omega(\Omega)} \leq \tfrac{C}{\alpha} \|f\|_{L^{2}(D)}\, \quad \text{and} \quad \|\nabla u_f^*\|_{L^{2}(\Omega)} \leq \tfrac{C}{\sqrt{\alpha\beta}} \|f\|_{L^{2}(D)}\,.
\] 
\end{proposition}
As a consequence, given \(f_1, f_2\in L^{2}(D)\), the corresponding minimizers \(u_{f_1}^*, u_{f_2}^{*}\) satisfy
\[
\|u_{f_1}^* - u_{f_2}^*\|_{L^2_\omega(\Omega)}
\le \tfrac{C}{\alpha} \|f_1 - f_2\|_{L^2(D)} \, \quad  \text{and} \quad \|\nabla u_{f_1}^*-\nabla u_{f_2}^*\|_{L^{2}(\Omega)} \leq \tfrac{C}{\sqrt{\alpha\beta}} \|f_1-f_2\|_{L^{2}(D)}\,.
\]
The continuous dependence of the \emph{minimizers} with respect to the targets \(f\) entails a similar property for the \emph{minimal values} of the functionals \(\cF^{(f)}_{\alpha, \beta}\):
\begin{corollary}\label{cor-proposition-stability}
Under the conditions of this section, let \(\alpha, \beta>0\) and $h$ be a continuous function satisfying~\eqref{eq:h-ass}. Given \(f_1, f_2\in L^{2}(D)\), let \(u_{f_1}^*\) and \(u_{f_2}^*\) be the  minimizers of \(\cF_{\alpha, \beta}^{(f_1)}\) and \(\cF_{\alpha, \beta}^{(f_2)}\). Then, there exists \(C=C(\Omega,D,\omega,h)>0\) such that
\[
|\cF^{(f_1)}_{\alpha, \beta}(u_{f_1}^*)-\cF^{(f_2)}_{\alpha, \beta}(u_{f_2}^*)|\leq \tfrac{C}{\alpha}(\|f_1\|_{L^{2}(D)}+\|f_2\|_{L^{2}(D)})\|f_1-f_2\|_{L^{2}(D)}\,.
\]
\end{corollary}

Let us also emphasize the local uniform stability property. We say that $u \in C^2(\overline \Omega)$ if $u,\; \nabla u,\; D^2 u$ admit continuous extensions to $
\overline{\Omega}$ and  $\|u\|_{C^2(\overline{\Omega})}
:=
\|u\|_{L^\infty(\Omega)}
+
\|\nabla u\|_{L^\infty(\Omega)}
+
\|D^2u\|_{L^\infty(\Omega)}<\infty.$
\begin{proposition}\label{prop:stabillity_L^infty}
Under the conditions of this section, let $\alpha,\beta>0$ and $h$ be a Borel function satisfying~\eqref{eq:h-ass} and \eqref{eq-Phi-loc-Lip}. For \(f\in L^{2}(D)\), let \(u_{f}^*\in \cW\) be the unique minimizer of \(\cF_{\alpha, \beta}^{(f)}\). Then, for every \(\Omega' \Subset \Omega\), there exists \(C=C(\Omega', \Omega, \alpha, \beta, D, \omega, h)>0\) such that $\|u_{f}^*\|_{C^{2}(\overline{\Omega'})}\leq C\|f\|_{L^{2}(D)}$.
%Assume further that $f_1,f_2 \in L^2(D)$, and $u_1^*,u_2^* \in \cW$ are the unique minimizers of 
%$\cF_{\alpha,\beta}^{(f_1)}$ and $\cF_{\alpha,\beta}^{(f_2)}$, respectively.
%Let $B_r \Subset B_R \Subset \Omega$. Then there exists $C = C(r,R,\alpha,\beta,D,h,\omega) > 0$, such that
%\[
%\|u_1^* - u_2^*\|_{L^\infty(B_r)}
%+
%\|\nabla u_1^* - \nabla u_2^*\|_{L^\infty(B_r)}
%\le
%C \, \|f_1 - f_2\|_{L^2(D)}.
%\]
\end{proposition}

It follows that for every \(f_1, f_2\in L^{2}(D)\) and every \(\Omega'\Subset \Omega\), the corresponding minimizers \(u^*_{f_1}\) and \(u^*_{f_2}\) satisfy 
\[
\|u^*_{f_1}-u^*_{f_2}\|_{C^{2}(\overline{\Omega'})} \leq C \|f_1-f_2\|_{L^{2}(D)}\,.\]

As a consequence of Proposition~\ref{prop:stability}, Proposition~\ref{prop:stabillity_L^infty} and the Fubini theorem, we get:

\begin{corollary}\label{cor:generalization}
Under the conditions of this section, let $f \in L^2(D)$, $\varsigma>0$, and let $f_\ve = f + \varepsilon$, where $\varepsilon$ is a
mean-zero random error with $\mathbb{E}\big[\|\varepsilon\|_{L^2(D)}^2\big] \le \varsigma^2$.
If $u_{f}^*,u_{f_\ve}^* \in \cW$ are the unique minimizers of
$\cF_{\alpha,\beta}^{(f)}$ and $\cF_{\alpha,\beta}^{(f_\ve)}$, then $
\mathbb{E}\big[\|u_{f_\ve}^* - u_f^*\|_{L^2_\omega(\Omega)}^2\big]
\le \tfrac{c^2}{\alpha^2}\,\varsigma^2$ for some $c=c(\Omega,D,\omega,h)>0$. Moreover, for every $\Omega' \Subset  \Omega$, there exists $C = C(\Omega', \Omega,\alpha,\beta,D,\omega,h)>0$ such that
\[
\mathbb{E}\Big[
\|u_{f_\ve}^* - u_f^*\|_{C^{2}(\overline{\Omega'})}^2
\Big]
\le
C^2 \varsigma^2.
\]
\end{corollary}

Propositions~\ref{prop:stability} and~\ref{prop:stabillity_L^infty}, together 
with Corollary~\ref{cor:generalization}, establish Lipschitz dependence of the minimizer on the target, with generalization error scaling 
linearly in the noise level $\varsigma$ and controlled by $1/\alpha$. The 
near-$C^3$ regularity ensures that the local geometry of the learned 
representation varies smoothly with the data, providing a concrete form of 
implicit bias toward smooth, well-generalizing solutions.

In order to describe the target of the neural network related to $\cF_{\alpha,\beta}^{(f)}$, we consider the gradient flow of this functional in the Hilbert space $L^2_\omega(\Omega)$, following \cite%[Definition 11.5]
{AmBrSe1}. This approach is natural in the analysis of overparametrized neural networks because the gradient flow corresponds to the continuous-time limit of gradient descent (the steepest descent curve of $\cF^{(f)}_{\alpha,\beta}$ in the $L^2_\omega$-geometry). It provides a powerful analytic framework for our novel study of convergence, implicit bias, and generalization properties that are often difficult to capture in discrete time.

Since \(\cF^{(f)}_{\alpha, \beta}\) is quadratic, the existence of the gradient flow follows from the Hille--Yosida theory. We introduce the unbounded linear operator   \(A:D(A)\subset L^{2}_\omega(\Omega)\to L^{2}_\omega(\Omega)\)   defined by its domain
\[
D(A):=\left\lbrace u\in \cW : \exists p=p(u)\in L^{2}_{\omega}(\Omega) \textrm{ such that } \forall v\in \cW, \int_{\Omega}\nabla u \cdot \nabla v \dth= -\int_{\Omega}pv\omega\dth\right\rbrace\,
\]
and $A u:=2\alpha u -2\beta p(u)+\frac{2}{\omega}\int_{\Omega}K(\theta, \cdot)u(\theta)\dth$. 
Note that here the function \(p(u)\) is \(\Delta u/\omega\).

It follows from the Hille--Yosida theorem complemented by the Br\'ezis--Komura theorem that:

\begin{proposition}\label{prop-Hille-Yosida-Brezis-Komura}
Under the conditions of this section,
for every \(u_0\in L^{2}_{\omega}(\Omega)\), there exists a unique \(u\in C^{0}([0,\infty[;L^{2}_\omega(\Omega))\cap C^1((0,\infty);L^{2}_\omega(\Omega))\cap C^0((0,\infty);D(A))\) such that 
\begin{equation}\label{eq13070}
\begin{cases}
\frac{du}{dt}=  -A u+2\frac{Q}{\omega}\qquad \forall t>0\,,\\
u(0)=u_0\,.
\end{cases}
\end{equation}
Moreover,  there exists a continuous semigroup of contractions \((S_t)_{t\geq 0}\) such that $u(t)=S_t u(0)$, and
\[
\forall v_1, v_2\in L^{2}_\omega(\Omega), \qquad \|S_t v_1 -S_t v_2\|_{L^{2}_\omega(\Omega)} \leq e^{-2\alpha t}\|v_1-v_2\|_{L^{2}_\omega(\Omega)}\,.
\]
Finally, \(t\mapsto \|u'(t)\|_{L^{2}_\omega(\Omega)}\) is nonincreasing on \((0,\infty)\) and
\[
\cF^{(f)}_{\alpha, \beta}(u(t))\leq \inf_{v\in \cW}\left(\cF^{(f)}_{\alpha, \beta}(v)+\frac{\alpha}{e^{2\alpha t}-1} \|u(0)-v\|_{L^{2}_\omega(\Omega)}\right)\,.
\]
\end{proposition}

%For $\varphi \in \mathrm{Dom}(\cF_{\alpha,\beta})$, we define the subdifferential through a variational inequality as in~\eqref{subdiff} and, when $\cF_{\alpha,\beta}$ is $\lambda$-convex, this condition can be written as
%\[
%\cF_{\alpha,\beta}(\psi)
%\ge
%\cF_{\alpha,\beta}(\varphi)
%+ \langle p, \psi - \varphi \rangle
%+ \tfrac{\lambda}{2} \|\psi - \varphi\|_{L^2_\omega}^2
%\quad \text{for all } \psi \in L^2_\omega(\Omega)\,.
%\]
%A curve $\varrho : (0,\infty) \to L^2_\omega(\Omega)$ is called a gradient flow of $\cF_{\alpha,\beta}$ if it is locally absolutely continuous and satisfies
%\[
%\partial_t \varrho(t) \in \partial \cF_{\alpha,\beta}(\varrho(t))
%\quad \text{for a.e. } t>0\,.
%\]

%Using the Euler--Lagrange equation from Section~3, the gradient flow can be written formally as
%\begin{equation}
%\label{eq:gradient-flow-expanded}
%\partial_t  \varrho
%=
%\beta \Delta  \varrho
%- \alpha \omega(\theta)  \varrho
%-Q(\theta)+\int_{\Omega}K(\theta,\vt)\vr(\vt)\d\vt
%\,.
%\end{equation}

%The existence and uniqueness of gradient flows follow from the Brézis--Komura theorem (Theorem~\ref{thm:brezis-komura}). As a consequence, the flow exhibits quantified stability as it converges exponentially fast toward the minimizer. 

\begin{corollary}
\label{cor:exp-convergence}Under the conditions of this section, let {$v\in L^{2}_\omega(\Omega)$} and $u_f^*$ be a minimizer of $\cF^{(f)}_{\alpha,\beta}$. 
%If $\varrho$ is the gradient flow  of $\cF^{(f)}_{\alpha,\beta}$ with $\varrho(0)=v$, 
Then {$\|S_t v - u_f^*\|_{L^2_\omega}
\le
e^{-2\alpha t}
\|v - u_f^*\|_{L^2_\omega}$}.
\end{corollary}
In the case of shallow networks, one can approximate the minimizer to $\cF^{(f)}_{\alpha,\beta}$ by considering the restriction of the functional to a finite dimensional space $E_\ell := \mathrm{span}\{e_0,\dots,e_\ell\} \subset \mathcal W$, where $\{e_i\}_{i\ge 0}$ forms a complete basis of the Hilbert space $L^2_\omega(\Omega)$. The problem thus reduces to solving a linear system, see Section~\ref{sec:numerics}.  In order to pass to the multi-layer case, the problem starts to be strongly nonlinear and analytically infeasible. Hence, another learning algorithm is necessary. One of the ideas could be to adapt the implicit Euler scheme associated to the gradient flow, see e.g. \cite[Section 12.2]{AmBrSe1}. Given $\tau>0$ and an initial condition $u\in \cW$, we define 
\begin{equation}
\label{eq:jko}
\tilde{\varrho}^0 = u,
\qquad
\tilde{\varrho}^{k\tau}
=
\arg\min_{v \in L^2_\omega(\Omega)}
\left\{
\cF_{\alpha,\beta}^{(f)}(v)
+ \tfrac{1}{2\tau}
\|v - \tilde{\varrho}^{(k-1)\tau}\|_{L^2_\omega}^2
\right\}\,.
\end{equation}
 The scheme admits minimizers due to the convexity and the lower semicontinuity of $\cF_{\alpha,\beta}^{(f)}$. Defining the piecewise constant interpolation $\tilde{\varrho}_\tau(t)$ by \eqref{eq-piecewie-interpolation}, one obtains convergence toward the gradient flow as $\tau \to 0$, cf. Theorem~\ref{thm:jko}. Combining this with Corollary~\ref{cor:exp-convergence} yields the error estimate:
\begin{corollary}
\label{cor:jko-error}Under the conditions of this section, let $u \in \mathcal W$ and $u_{f}^*$ be the minimizer  of $\cF^{(f)}_{\alpha,\beta}$. Then, for any \(\tau>0\) and \(t\geq 0\),
\[
\|\tilde{\varrho}_\tau(t) - u_{f}^*\|_{L^2_\omega}
\le
e^{-2\alpha t}\|u - u_{f}^*\|_{L^2_\omega}
+ 2(\sqrt{2}+1) \sqrt{\tau \, \cF^{(f)}_{\alpha,\beta}(u)}\,.
\]
\end{corollary}

\section{Numerical experiments}\label{sec:numerics}

The experiments below illustrate the key properties established in the theoretical analysis: no Lavrentiev gap, $\lambda$-convexity, regularity of the optimal parameter density. We note that our numerical findings show that estimations of $\cF^{(f)}_{\alpha,\beta}$ are both accurate and robust to noise and outliers. All simulations are done on standard personal computer and can  be found in the supplementary material. %\url{https://github.com/192459/variational-formulation-nn}.

\textbf{Reduction to a linear system.} We approximate the parameter density $u \in \cW$ by a polynomial ansatz $u(\theta) \approx \sum_{i=1}^M a_i\,\widehat{u}_i(\theta).
% \qquad
% \widehat{u}_i(\theta) = \prod_{j=1}^{d+1}\theta_j^{p_i^j},
% \quad p_i^j\in\mathbb{N}_0\,,
$
We have used three different types of basis functions $\widehat{u}_i$: polynomials, orthonormal cosine basis and Legendre polynomials. We minimize $\cF_{\alpha,\beta}^{(f)}$ over the coefficient vector $\vec a\in\R^M$. 
 Since the functional $\cF_{\alpha,\beta}^{(f)}$ involves expectations with respect to the (unknown) data distribution, it cannot be evaluated directly. We therefore work with its empirical version, obtained by replacing the population quantities with averages over the observed data.
Given data points $(x_k, f(x_k))_{k=1}^{N_D}$ , the functional reduces to the sum of a
quadratic form and an affine term:
\[
\widehat{\cF}_{\alpha,\beta}^{\ (f)}(a)
= C_D\left( \vec f\,^\top \vec f - 2\vec f\,^\top U\vec a + \vec a^\top U^\top U \,\vec a\right)+{\vec a}^\top(\alpha V + \beta W)\,\vec a\,,
\]
where $C_D=\frac{\mathcal{L}^d(D)}{N_D}$ (here, $\mathcal{L}^d(D)$ is the Lebesgue measure of \(D\) ), $\vec f = (f(x_1),\ldots,f(x_{N_D}))^\top$, $\vec a=(a_1,\ldots,a_M)^\top$, and 
the matrices
\[
U_{ki} = \int_\Omega h(\theta, x_k)\,\widehat{u}_i(\theta)\d\theta,
\quad
V_{ij} =  %\int_\Omega \widehat{u}_i(\theta)\,\widehat{u}_j(\theta)\,\omega(\theta)\d\theta,
\langle \widehat{u}_i , \widehat u_j \rangle_{L^2_\omega (\Omega)},
\quad
W_{ij} = \langle \nabla \widehat{u}_i, \nabla \widehat{u}_j \rangle _{L^2 (\Omega)}
\]
are determined by the given data, the parameter domain $\Omega$, and the weight $\omega(\theta)=1+|\theta|^{2d+4}$. Crucially, no iterative optimization or gradient descent is required. We can find the minimizer directly by solving the linear equation:
\begin{equation*}
\nabla \widehat{\cF}^{\ (f)}_{\alpha, \beta} (\vec a) %= 2(C_D U^\top U + \alpha V+\beta W)\vec a -  2C_D U^\top \vec  f 
= 
0 \implies \vec  a = \left(U^\top U +\frac{\alpha}{C_D} V+\frac{\beta}{C_D} W\right)^{-1}\left(U^\top \vec f\right) \, .
\end{equation*}
The matrix $U^\top U + \frac{\alpha}{C_D} V + \frac{\beta}{C_D} W$ is positive semidefinite (and moreover positive definite for $\alpha>0$ in the case the basis functions $\{\widehat{u}_i\}_{i=1}^{M}$ are linearly independent in $L^2_\omega(\Omega)$ or $\beta>0$ and the gradients of the basis functions are linearly independent in $L^{2}(\Omega)$) and the minimizer is found via pseudoinverse (or inverse); the problem reduces to \emph{ridge regression}.

Let us compare: (i) the unregularized case $\alpha=\beta=0$, which corresponds to the least-squares projection onto the polynomial ansatz; and (ii) a single-hidden-layer network $f_N$ with $N=10{\,}000$ ReLU neurons trained by Adam optimizer \cite{ADAM}, representing the finite-width benchmark. In Example~4 NN was trained via SGD due to the instability of Adam.

\paragraph{Example 1: Sinus function for $\boldsymbol{d=1}$. } 

We generate $N_D = 50$ observations of the function $\sin(7x)$ plus small Gaussian noise on the interval $(-1,1)$. The minimum of $\cF^{(f)}_{\alpha, \beta}$ is approximated by a polynomial \eqref{eq:polynomial} of order $s=15$, thus $M = 136$. We have chosen  {$\Omega=(-R,R)\times (-L,L)$ with $R=L=7$}. See the results on Figure~\ref{fig:example1}~(a) and compare the prediction for the single-layer neural network with $N=10\,000$ neurons (NN), our prediction for $\alpha=\beta=0$ and then with small penalties $\alpha=8.8\times10^{-12}$, $\beta=8.8\times 10^{-10}$.

\begin{figure}[h!] 
\centering 
\begin{subfigure}[b]{0.49\textwidth} \centering \includegraphics[width=\textwidth]{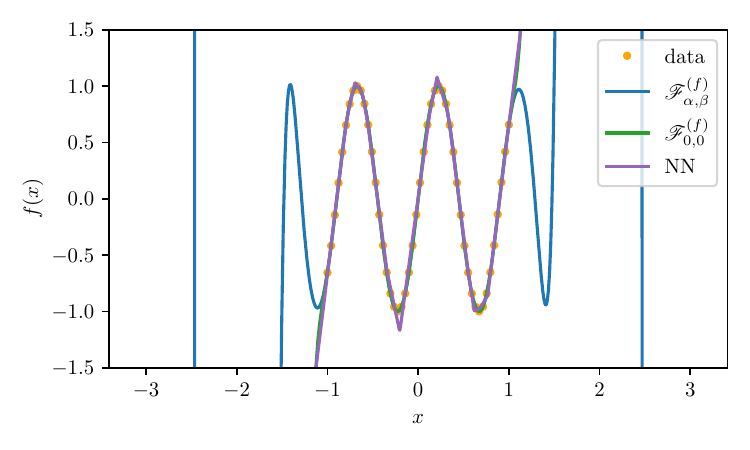} \caption{Example 1} 
\end{subfigure} 
\hfill 
\begin{subfigure}[b]{0.49\textwidth} 
\centering 
\includegraphics[width=\textwidth]{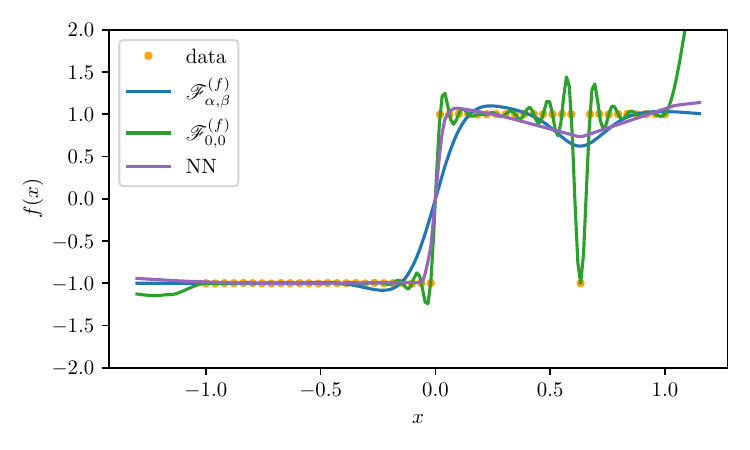} \caption{Example 2} 
\end{subfigure} 
\caption{Prediction for functions $\mathrm{sin}$ and perturbed $\mathrm{sign}$ in $d=1$} 
\label{fig:example1} 
\end{figure} 

%Figure~\ref{fig:example1}(a) shows that the regularized solution tracks the target accurately and smoothly, while the unregularized ansatz ($\alpha=\beta=0$) overfits the noise and produces a comparable but noisier fit.  The smooth reconstruction is consistent with the near-$C^3$ regularity of the minimizer (Theorem~\ref{thm:convexity}): the optimal parameter density is far from oscillatory. The learned function concentrates near a low-dimensional smooth manifold rather than memorizing individual data points, illustrating the non-overfitting mechanism identified in the introduction.

Figure~\ref{fig:example1}(a) shows that the regularized solution tracks the target accurately and smoothly, while the unregularized ansatz overfits and the neural network baseline is noisier. The Sobolev penalty makes the optimal density far from oscillatory, producing a function concentrated near a smooth manifold reflecting near-$C^3$ regularity of Theorem~\ref{thm:convexity} and the implicit bias toward smooth solutions. In addition, surprisingly, the regularized version is able to recover almost fully the next period of the $\sin$ function.

\paragraph{Example 2: Discontinuous target --- sign function ($d=1$)}

We generate $N_D = 50$ observations of $\mathrm{sign}(x)$ on $(-1,1)$ with small Gaussian noise, 
including one outlier, and use  trigonometric polynomials $\{\widehat{u}_i\}_{i=1}^{M}$ \eqref{eq:harmonic} up to the coefficient $s=60$ ($M=1{,}891$). We have chosen \(\Omega=(-R,R)\times (-L,L)\) with $R=5$, $L = 5.1$. For the regularized functional, we have taken $\alpha = 4\times 10^{-8}$ and $\beta = 2\times 10^{-7}$.

Figure~\ref{fig:example1}(b) illustrates the stability stated in
Propositions~\ref{prop:stability}--\ref{prop:stabillity_L^infty}: the minimizer 
varies continuously with the data (Corollary~\ref{cor:generalization}), avoiding 
the large excursions of the unregularized case, while the no-Lavrentiev-gap 
property (Theorem~\ref{theo:no-Lavr}) guarantees that the polynomial 
approximation incurs no hidden error. Additional pictures are given in Appendix, see Figures~\ref{fig3} and \ref{fig4}.

\paragraph{Examples 3--4: Benchmark datasets ($d>1$)} 

We validate the approach on two standard regression benchmarks, comparing against 
the single-hidden-layer network with $\mathrm{ReLU}$ activation function with  $N=10{\,}000$ neurons. In both examples we have split the dataset into the train and test with ratio of test to be $0.2$. 

In Example~3, we have used Diabetes dataset \cite{lars} containing $d=10$ features with $N_D=442$ observations. We have taken \(\Omega=B_{1}^{2}\times (-1,1)\) with \(B_{1}^2\) the unit ball in \(\R^2\),  and $s=5$ implying $M=4{,}368$. The sparsity of the matrix $U$ is $64\,\%$. We have chosen very small coefficients $\alpha/C_D = \beta/C_D =10^{-10}$.  

In  Example~4, we have selected California housing dataset \cite{california_housing}. We take \(\Omega=B_{1}^{8}\times (-1,1)\) (where \(B_{1}^{8}\) is the unit ball in \(\R^8\)), $\alpha/C_D = \beta/C_D =10^{-10}$, polynomials of order $s = 6$, so that $M = 5{,}005$. The sparsity of the matrix $U$ is $34\,\%$. We have selected single category of ocean proximity (near bay) leading to $N_D = 1{,}860$ observations with $d=8$ features. 

\begin{table}[h]
\caption{Results of approximation in Examples~3 and~4. We present Root Mean Square Error (RMSE), Mean Absolute Error (MAE) and ${R}^2$ coefficient.}
\label{tab}
\centering
\begin{tabular}{lrrrrrr} \toprule
& \multicolumn{3}{c}{Example 3} & \multicolumn{3}{c}{Example 4} \\  \midrule
Metric & \multicolumn{1}{c}{$\widehat{\mathscr{F}}^{\ (f)}_{0, 0}$}& \multicolumn{1}{c}{$\widehat{\mathscr{F}}^{\ (f)}_{\alpha, \beta}$} &\multicolumn{1}{c}{NN} & \multicolumn{1}{c}{$\widehat{\mathscr{F}}^{\ (f)}_{0, 0}$} & {$\widehat{\mathscr{F}}^{\ (f)}_{\alpha, \beta}$} & \multicolumn{1}{c}{NN} \\ \midrule
RMSE &  $51.29$ & $52.00$ & $63.28$ & $0.38$ & $0.55$&$0.29$\\
MAE & $40.54$ & $40.39$ &$48.43$ & $0.62$ & $0.39$ &$0.54$\\
$R^2$ & $0.50$ & $0.49$ & $0.24$ &  $0.73$ & $0.78$ & $0.79$\\ \bottomrule
\end{tabular}
\end{table}

The results presented in Table~\ref{tab} confirm that the regularized functional 
achieves competitive or superior accuracy, aligning with Theorem~\ref{theo:no-Lavr}. This yields strong finite-sample 
performance.

\section{Conclusion}

The right geometry unlocks the right tools. Casting shallow network training as a variational problem over parameter densities in a weighted Sobolev space gives simultaneous access to global $\lambda$-convexity, elliptic regularity, and gradient flow theory, which is a combination unavailable in either the NTK or Wasserstein frameworks. The outcome is a continuum model that is analytically controlled, dynamically stable, and exactly tied to finite-width networks via an explicit $O(1/N)$ bound: no approximation error, no linearization, no weak regularity.

The near-$C^3$ smoothness of optimal parameter densities is perhaps the most surprising consequence. It suggests that overparameterized networks are implicitly attracted to low-dimensional smooth structures in parameter space. This is a concrete, quantitative form of implicit bias that goes beyond what existing analyses can see. To be stressed out, it is not an artifact of the limit: the absence of a Lavrentiev gap ensures that the continuum picture faithfully reflects the discrete one. Concretely, the minimizer is found via a single ridge regression, with generalization error bounded by $C/\alpha$, making $\alpha$ a directly interpretable robustness parameter rather than a black-box 
regularizer.

The shallow setting was the necessary first step. Resolving convexity, regularity, and convergence cleanly here opens a concrete path toward multilayer architectures, sharper implicit bias descriptions, and new regularization principles grounded in the underlying variational framework.

\section*{Acknowledgments}
MB was supported by internal grant for specific research No. FSI-S-26-8958. IC and BM were supported by National Science Centre Grant 2024/55/B/ST6/016.

\bibliographystyle{plainnat}
\bibliography{neurips_intro}

\begin{thebibliography}{35}
\providecommand{\natexlab}[1]{#1}
\providecommand{\url}[1]{\texttt{#1}}
\expandafter\ifx\csname urlstyle\endcsname\relax
  \providecommand{\doi}[1]{doi: #1}\else
  \providecommand{\doi}{doi: \begingroup \urlstyle{rm}\Url}\fi

\bibitem[Allen-Zhu et~al.(2019)Allen-Zhu, Li, and Song]{allenzhu2019convergence}
Zeyuan Allen-Zhu, Yuanzhi Li, and Zhao Song.
\newblock A convergence theory for deep learning via over-parameterization.
\newblock In \emph{Proceedings of the 36th International Conference on Machine Learning}, volume~97 of \emph{Proceedings of Machine Learning Research}, pages 242--252, 2019.

\bibitem[Ambrosio et~al.(2008)Ambrosio, Gigli, and Savar{\'e}]{ambrosio2008gradient}
Luigi Ambrosio, Nicola Gigli, and Giuseppe Savar{\'e}.
\newblock \emph{Gradient Flows: In Metric Spaces and in the Space of Probability Measures}.
\newblock Birkh{\"a}user, 2008.

\bibitem[Ambrosio et~al.(2021)Ambrosio, Bru\'e, and Semola]{AmBrSe1}
Luigi Ambrosio, Elia Bru\'e, and Daniele Semola.
\newblock \emph{Lectures on optimal transport}, volume 130 of \emph{Unitext}.
\newblock Springer, Cham, 2021.
\newblock ISBN 978-3-030-72161-9; 978-3-030-72162-6.
\newblock \doi{10.1007/978-3-030-72162-6}.
\newblock URL \url{https://doi.org/10.1007/978-3-030-72162-6}.
\newblock La Matematica per il 3+2.

\bibitem[Arora et~al.(2019)Arora, Cohen, Hu, and Luo]{Arora2019Implicit}
Sanjeev Arora, Nadav Cohen, Wei Hu, and Yuping Luo.
\newblock Implicit regularization in deep matrix factorization.
\newblock \emph{Advances in Neural Information Processing Systems}, 32, 2019.

\bibitem[Bach(2017)]{bach2017breaking}
Francis Bach.
\newblock Breaking the curse of dimensionality with convex neural networks.
\newblock \emph{Journal of Machine Learning Research}, 18\penalty0 (19):\penalty0 1--53, 2017.

\bibitem[Balci et~al.(2020)Balci, Diening, and Surnachev]{badisu}
Anna~Kh. Balci, Lars Diening, and Mikhail Surnachev.
\newblock New examples on {L}avrentiev gap using fractals.
\newblock \emph{Calc. Var. Partial Differential Equations}, 59\penalty0 (5):\penalty0 Paper No. 180, 34, 2020.
\newblock ISSN 0944-2669,1432-0835.
\newblock \doi{10.1007/s00526-020-01818-1}.
\newblock URL \url{https://doi.org/10.1007/s00526-020-01818-1}.

\bibitem[Belkin et~al.(2019)Belkin, Hsu, Ma, and Mandal]{Belkin2019Reconciling}
Mikhail Belkin, Daniel Hsu, Siyuan Ma, and Soumik Mandal.
\newblock Reconciling modern machine-learning practice and the classical bias--variance trade-off.
\newblock \emph{Proceedings of the National Academy of Sciences}, 116\penalty0 (32):\penalty0 15849--15854, 2019.
\newblock \doi{10.1073/pnas.1903070116}.

\bibitem[Borowski et~al.(2024)Borowski, Chlebicka, De~Filippis, and Miasojedow]{bcdfm}
Micha\l{} Borowski, Iwona Chlebicka, Filomena De~Filippis, and B\l{a}\.{z}ej Miasojedow.
\newblock Absence and presence of {L}avrentiev's phenomenon for double phase functionals upon every choice of exponents.
\newblock \emph{Calc. Var. Partial Differential Equations}, 63\penalty0 (2):\penalty0 Paper No. 35, 23, 2024.
\newblock ISSN 0944-2669,1432-0835.
\newblock \doi{10.1007/s00526-023-02640-1}.
\newblock URL \url{https://doi.org/10.1007/s00526-023-02640-1}.

\bibitem[Brézis(2011)]{Brezis}
Haïm Brézis.
\newblock \emph{Functional Analysis, {S}obolev Spaces and Partial Differential Equations}.
\newblock Universitext. Springer, New York, 2011.
\newblock ISBN 978-0-387-70913-0.

\bibitem[Chen et~al.(2024)Chen, Liu, Lu, Chrysos, and Cevher]{chen2024resnet}
Yihang Chen, Fanghui Liu, Yiping Lu, Grigorios~G. Chrysos, and Volkan Cevher.
\newblock Generalization of scaled deep resnets in the mean-field regime.
\newblock \emph{arXiv preprint arXiv:2403.09889}, 2024.

\bibitem[Chizat and Bach(2018)]{chizat2018global}
L{\'e}na{\"i}c Chizat and Francis Bach.
\newblock On the global convergence of gradient descent for over-parameterized models using optimal transport.
\newblock In \emph{Advances in Neural Information Processing Systems}, 2018.

\bibitem[Chizat et~al.(2025)Chizat, Colombo, and Fernández-Real]{chizat2025convergencedriftdiffusionpdesarising}
Lénaïc Chizat, Maria Colombo, and Xavier Fernández-Real.
\newblock Convergence of drift-diffusion pdes arising as wasserstein gradient flows of convex functions, 2025.
\newblock URL \url{https://arxiv.org/abs/2507.12385}.

\bibitem[Dereich et~al.(2024)Dereich, Jentzen, and Kassing]{dereich2024minimizers}
Steffen Dereich, Arnulf Jentzen, and Sebastian Kassing.
\newblock On the existence of minimizers in shallow residual {ReLU} neural network optimization landscapes.
\newblock \emph{SIAM Journal on Numerical Analysis}, 62\penalty0 (6):\penalty0 2640--2666, 2024.
\newblock \doi{10.1137/23M1556241}.

\bibitem[Du et~al.(2019)Du, Lee, Li, Wang, and Zhai]{du2019gradient}
Simon~S. Du, Jason~D. Lee, Haochuan Li, Liwei Wang, and Xiyu Zhai.
\newblock Gradient descent finds global minima of deep neural networks.
\newblock In \emph{Proceedings of the 36th International Conference on Machine Learning}, volume~97 of \emph{Proceedings of Machine Learning Research}, pages 1675--1685, 2019.

\bibitem[E et~al.(2018)E, Han, and Li]{E2018}
Weinan E, Jiequn Han, and Qianxiao Li.
\newblock A mean-field optimal control formulation of deep learning.
\newblock \emph{Research in the Mathematical Sciences}, 6\penalty0 (1), December 2018.
\newblock ISSN 2197-9847.
\newblock \doi{10.1007/s40687-018-0172-y}.
\newblock URL \url{http://dx.doi.org/10.1007/s40687-018-0172-y}.

\bibitem[E et~al.(2020)E, Ma, and Wu]{E2020}
Weinan E, Chao Ma, and Lei Wu.
\newblock Machine learning from a continuous viewpoint, {I}.
\newblock \emph{Science China Mathematics}, 63\penalty0 (11):\penalty0 2233--2266, Nov 2020.
\newblock ISSN 1869-1862.
\newblock \doi{10.1007/s11425-020-1773-8}.
\newblock URL \url{https://doi.org/10.1007/s11425-020-1773-8}.

\bibitem[Efron et~al.(2004)Efron, Hastie, Johnstone, and Tibshirani]{lars}
Bradley Efron, Trevor Hastie, Iain Johnstone, and Robert Tibshirani.
\newblock Least angle regression.
\newblock \emph{Ann. Statist.}, 32\penalty0 (2):\penalty0 407--499, 2004.
\newblock ISSN 0090-5364,2168-8966.
\newblock \doi{10.1214/009053604000000067}.
\newblock URL \url{https://doi.org/10.1214/009053604000000067}.
\newblock With discussion, and a rejoinder by the authors.

\bibitem[Fern\'{a}ndez-Real and Figalli(2022)]{FRF}
Xavier Fern\'{a}ndez-Real and Alessio Figalli.
\newblock The continuous formulation of shallow neural networks as {W}asserstein-type gradient flows.
\newblock In \emph{Analysis at large---dedicated to the life and work of {J}ean {B}ourgain}, pages 29--57. Springer, Cham, 2022.
\newblock ISBN 978-3-031-05330-6; 978-3-031-05331-3.
\newblock \doi{10.1007/978-3-031-05331-3_3}.
\newblock URL \url{https://doi.org/10.1007/978-3-031-05331-3_3}.

\bibitem[Fonseca et~al.(2004)Fonseca, Mal\'y, and Mingione]{fomami}
Irene Fonseca, Jan Mal\'y, and Giuseppe Mingione.
\newblock Scalar minimizers with fractal singular sets.
\newblock \emph{Arch. Ration. Mech. Anal.}, 172\penalty0 (2):\penalty0 295--307, 2004.
\newblock ISSN 0003-9527,1432-0673.
\newblock \doi{10.1007/s00205-003-0301-6}.
\newblock URL \url{https://doi.org/10.1007/s00205-003-0301-6}.

\bibitem[Gilbarg and Trudinger(2001)]{GT}
David Gilbarg and Neil~S. Trudinger.
\newblock \emph{Elliptic partial differential equations of second order}.
\newblock Class. Math. Berlin: Springer, reprint of the 1998 ed. edition, 2001.
\newblock ISBN 3-540-41160-7.

\bibitem[Jacot et~al.(2018)Jacot, Gabriel, and Hongler]{jacot2018ntk}
Arthur Jacot, Franck Gabriel, and Cl{\'e}ment Hongler.
\newblock Neural tangent kernel: Convergence and generalization in neural networks.
\newblock In \emph{Advances in Neural Information Processing Systems}, 2018.

\bibitem[Jordan et~al.(1998)Jordan, Kinderlehrer, and Otto]{JKO}
Richard Jordan, David Kinderlehrer, and Felix Otto.
\newblock The variational formulation of the {F}okker-{P}lanck equation.
\newblock \emph{SIAM J. Math. Anal.}, 29\penalty0 (1):\penalty0 1--17, 1998.
\newblock ISSN 0036-1410,1095-7154.
\newblock \doi{10.1137/S0036141096303359}.
\newblock URL \url{https://doi.org/10.1137/S0036141096303359}.

\bibitem[Kingma and Ba(2014)]{ADAM}
Diederik Kingma and Jimmy Ba.
\newblock Adam: A method for stochastic optimization.
\newblock \emph{International Conference on Learning Representations}, 12 2014.

\bibitem[Mao et~al.(2026)Mao, Siegel, and Xu]{gribonval2025shallow}
Tong Mao, Jonathan~W. Siegel, and Jinchao Xu.
\newblock Approximation by shallow neural networks with {ReLU}$^k$ activation: {S}obolev spaces and optimal rates via the {R}adon transform.
\newblock \emph{SIAM Journal on Mathematical Analysis}, 58\penalty0 (2):\penalty0 1171--1186, 2026.
\newblock \doi{10.1137/24M1686693}.
\newblock URL \url{https://doi.org/10.1137/24M1686693}.

\bibitem[Mei et~al.(2018)Mei, Montanari, and Nguyen]{mei2018mean}
Song Mei, Andrea Montanari, and Phan-Minh Nguyen.
\newblock A mean field view of the landscape of two-layer neural networks.
\newblock \emph{Proceedings of the National Academy of Sciences}, 115\penalty0 (33):\penalty0 E7665--E7671, 2018.

\bibitem[Mingione and R\v{a}dulescu(2021)]{mira}
Giuseppe Mingione and Vicen\c{t}iu R\v{a}dulescu.
\newblock Recent developments in problems with nonstandard growth and nonuniform ellipticity.
\newblock \emph{J. Math. Anal. Appl.}, 501\penalty0 (1):\penalty0 Paper No. 125197, 41, 2021.
\newblock ISSN 0022-247X,1096-0813.
\newblock \doi{10.1016/j.jmaa.2021.125197}.
\newblock URL \url{https://doi.org/10.1016/j.jmaa.2021.125197}.

\bibitem[Mousavi-Hosseini et~al.(2025)Mousavi-Hosseini, Wu, and Erdogdu]{mousavi2025multiindex}
Alireza Mousavi-Hosseini, Denny Wu, and Murat~A. Erdogdu.
\newblock Learning multi-index models with neural networks via mean-field {L}angevin dynamics.
\newblock In \emph{International Conference on Learning Representations}, 2025.

\bibitem[Nitanda(2024)]{nitanda2024particle}
Atsushi Nitanda.
\newblock Improved particle approximation error for mean field neural networks.
\newblock In \emph{Advances in Neural Information Processing Systems}, 2024.

\bibitem[Oh et~al.(2025)Oh, Lyu, and Son]{oh2025sobolev}
Jong~Kwon Oh, Hanbaek Lyu, and Hwijae Son.
\newblock Sobolev acceleration for neural networks.
\newblock \emph{arXiv preprint arXiv:2509.19773}, 2025.

\bibitem[Ongie et~al.(2020)Ongie, Willett, Soudry, and Srebro]{ongie2019function}
Greg Ongie, Rebecca Willett, Daniel Soudry, and Nathan Srebro.
\newblock A function space view of bounded norm infinite width relu nets: The multivariate case.
\newblock In \emph{8th International Conference on Learning Representations, {ICLR} 2020, Addis Ababa, Ethiopia, April 26-30, 2020}. OpenReview.net, 2020.
\newblock URL \url{https://openreview.net/forum?id=H1lNPxHKDH}.

\bibitem[Pedregosa et~al.(2011)Pedregosa, Varoquaux, Gramfort, et~al.]{california_housing}
Fabian Pedregosa, Ga{\"e}l Varoquaux, Alexandre Gramfort, et~al.
\newblock Scikit-learn: Machine learning in python.
\newblock \emph{Journal of Machine Learning Research}, 12:\penalty0 2825--2830, 2011.

\bibitem[Razin and Cohen(2020)]{NEURIPS2020_f21e255f}
Noam Razin and Nadav Cohen.
\newblock Implicit regularization in deep learning may not be explainable by norms.
\newblock In H.~Larochelle, M.~Ranzato, R.~Hadsell, M.F. Balcan, and H.~Lin, editors, \emph{Advances in Neural Information Processing Systems}, volume~33, pages 21174--21187. Curran Associates, Inc., 2020.
\newblock URL \url{https://proceedings.neurips.cc/paper_files/paper/2020/file/f21e255f89e0f258accbe4e984eef486-Paper.pdf}.

\bibitem[Rotskoff and Vanden-Eijnden(2022)]{rotskoff2018trainability}
Grant~M. Rotskoff and Eric Vanden-Eijnden.
\newblock Trainability and accuracy of neural networks: An interacting particle system approach.
\newblock \emph{Communications on Pure and Applied Mathematics}, 75\penalty0 (9):\penalty0 1889--1935, 2022.

\bibitem[Santambrogio(2017)]{Santambrogio-survey}
Filippo Santambrogio.
\newblock \{{E}uclidean, metric, and {W}asserstein\} gradient flows: an overview.
\newblock \emph{Bull. Math. Sci.}, 7\penalty0 (1):\penalty0 87--154, 2017.
\newblock ISSN 1664-3607,1664-3615.
\newblock \doi{10.1007/s13373-017-0101-1}.
\newblock URL \url{https://doi.org/10.1007/s13373-017-0101-1}.

\bibitem[Sirignano and Spiliopoulos(2020)]{sirignano2020mean}
Justin Sirignano and Konstantinos Spiliopoulos.
\newblock Mean field analysis of neural networks: A law of large numbers.
\newblock \emph{SIAM Journal on Applied Mathematics}, 80\penalty0 (2):\penalty0 725--752, 2020.

\end{thebibliography}

%%%%%%%%%%%%%%%%%%%%%%%%%%%%%%%%%%%%%%%%%%%%%%%%%%%%%%%%%%%%

\newpage 
\appendix

\counterwithin{theorem}{section}
\renewcommand{\thetheorem}{\thesection.\arabic{theorem}}
\setcounter{theorem}{0}

\renewcommand{\theequation}{A.\arabic{equation}}
\setcounter{equation}{0}

\section{Technical appendices and supplementary material}
%Technical appendices with additional results, figures, graphs, and proofs may be submitted with the paper submission before the full submission deadline (see above). You can upload a ZIP file for videos or code, but do not upload a separate PDF file for the appendix. There is no page limit for the technical appendices. 

%Note: Think of the appendix as ``optional reading'' for reviewers. The paper must be able to stand alone without the appendix; for example, adding critical experiments that support the main claims to an appendix is inappropriate. 

In the case it is clear from the context, we simplify the notation setting $\cF_{\alpha,\beta}=\cF^{(f)}_{\alpha,\beta}$.

\subsection{Proof of Theorem~\ref{theo:no-Lavr}}

Theorem~\ref{theo:no-Lavr} is a consequence of two approximation results. In the first one (Proposition~\ref{prop-no-gap-M1-Msm}), we regularize any \(m\in \cM_1(\Omega)\) by relying on standard convolution and truncation arguments. We then prove that the functional \(\cR(f,\cdot)\) converges along the regularized sequence to \(\cR(f,m)\) by exploiting the growth assumptions satisfied by \(h\). In the second approximation result (Proposition~\ref{prop-no-gap-M1-Mat}), we approximate \(m\) by a sequence of purely atomic measures \(\mu_i\). In order to establish the convergence of \(\cR(f,\mu_i)\) to \(\cR(f,m)\), we strongly rely on the quadratic structure of \(\cR(f,\cdot)\).

\begin{proposition}\label{prop-no-gap-M1-Msm}
For every \(m\in \cM_1(\Omega)\), there exists a sequence \((u_j)_{j\geq 1} \subset C^{\infty}_c(\Omega)\) such that 
\[
\lim_{j\to +\infty}\cR(f,u_j\dx) = \cR(f,m)\,.
\]
\end{proposition}

\begin{proof}
We first consider the case when \(m\) belongs to the set \(\cM_c(\Omega)\) of those finite Borel measures with compact support in \(\Omega\). 
We  extend  \(m\) as a measure on \(\R^{d+1}\) simply by assigning to any Borel set \(B\subset \R^{d+1}\) the value \(m(B\cap \Omega)\). 

Fix \(0<\epsilon_0<\textrm{dist}(\textrm{supp }m, \partial \Omega)\).
We then introduce a smooth regularization kernel \((\rho_{\epsilon})_{0<\epsilon<\epsilon_0}\) with \(\rho_{\epsilon}\in C^{\infty}_c(B_\epsilon)\), \(\rho_{\epsilon}\geq 0\) and \(\int_{\R^{d+1}}\rho_{\epsilon}=1\) for every \(0<\epsilon<\epsilon_0\). {Here, we have denoted by \(B_\epsilon\) the ball of center \(0\) and radius \(\epsilon\) in \(\R^{d+1}\).} 
We then define
\[
m_{\epsilon}:=m*\rho_{\epsilon}:x\in \Omega \mapsto \int_{\R^{d+1}} \rho_{\epsilon}(x-y)\d m(y)\,.
\]
Then, \(m*\rho_{\epsilon}\in C^{\infty}_c(\Omega)\), with \(\textrm{supp }m*\rho_{\epsilon}\subset \textrm{supp }m + B_{\epsilon}\). 
%Moreover, by the Fubini--Tonelli theorem\footnote{not true for signed measures, but that's not a problem as we have $\|m_\epsilon\|_{L^1}\leq  |m|(\Omega)$}, 
%\[
%\int_{\Omega}m_{\epsilon}(x)\dx = m(\Omega)\,.
%\]
We claim that 
\begin{equation}\label{eq372}
\lim_{\epsilon\to 0} \cR(f, m*\rho_{\epsilon}\dx) = \cR(f,m)\,.
\end{equation}
There exists a compact set  \(K\Subset \Omega\) such that \(\textrm{supp }m + B_{\epsilon}\subset K\) for every \(0<\epsilon<\epsilon_0\). 
Let \(\eta\in C^{\infty}_c(\Omega)\) be a cut-off function; that is, \(0\leq \eta\leq 1\)  and \(\eta\equiv 1\) on \(K\). Then, for every \(\kappa\in C(\Omega)\), we have
\[
\int_{\Omega}\kappa m_{\epsilon}\dth = \int_{\R^{d+1}}(\eta\kappa)m*\rho_{\epsilon}\dth\,.
\]
By the Fubini theorem,
\[
\int_{\Omega}\kappa m_{\epsilon}\dth = \int_{\R^{d+1}} (\eta\kappa)*\check{\rho}_{\epsilon}\d m\,,
\]
where \(\check{\rho}_\epsilon(\cdot)=\rho_\epsilon(-\cdot)\).
Since \(((\eta\kappa)*\check{\rho}_{\epsilon})_{\epsilon}\) uniformly converges to \(\eta\kappa\) on \(\Omega\) when \(\epsilon\to 0\), we deduce that
\[
\lim_{\epsilon\to 0}\int_{\Omega}\kappa m_{\epsilon}\dth =
\lim_{\epsilon\to 0}\int_{\R^{d+1}} (\eta\kappa)*\check{\rho}_{\epsilon}\d m=\int_{\R^{d+1}} \eta\kappa\d m\,.
\]
Since \(\eta\kappa=\kappa\) on \(\textrm{supp } m\), this gives
\begin{equation}\label{eq565}
\lim_{\epsilon\to 0}\int_{\Omega}\kappa m_{\epsilon}\dth =
\int_{\Omega}\kappa\d m\,.
\end{equation}
We can repeat the same argument with the measure \(\nu= m \otimes  m\) on \(\R^{d+1}\times \R^{d+1}\). More specifically, by the Fubini theorem, for every \(\overline{\kappa}\in C(\Omega\times\Omega)\), one has
\[
\int_{\Omega\times \Omega} \overline{\kappa}\d m_{\epsilon}\otimes \d m_{\epsilon}
=\int_{\Omega\times \Omega} (\widetilde{\kappa}*\widetilde{\rho}_{\epsilon})\d m\otimes \d m\,,
\]
where
\[
\widetilde{\kappa}(x,y):=\eta(x)\eta(y)\overline{\kappa}(x,y)\,, \qquad
\wt{\rho}_\epsilon(x,y):=\rho_\epsilon(-x)\rho_\epsilon(-y)\,.
\]
Since \(\widetilde{\kappa}\in C_c(\Omega\times\Omega)\), the integrand in the left-hand side uniformly converges to \(\widetilde{\kappa}\) when \(\epsilon\to 0\), which implies that
\begin{equation}\label{eq585}
\lim_{\epsilon\to 0}\int_{\Omega\times \Omega} \overline{\kappa} \d m_{\epsilon}\otimes \d m_{\epsilon}
=\int_{\Omega\times \Omega} \overline{\kappa}\d m\otimes \d m\,.
\end{equation}

Recall the decomposition of $\cR$ from~\eqref{eq:quadratic-decomposition} involving
 \(K\) and \(Q\) given by \eqref{eq:bK}. Since $h$ satisfies~\eqref{eq:h-ass}, for every  \(L>0\) and for every \(\theta, \vt\in B_L\cap \Omega\) it holds
\[
|h(\theta,x)h(\vt,x)|\leq C_{h}^2 (1+L)^2(1+|x|)^2\quad \text{and} \quad |f(x)h(\theta,x)|\leq C_{h}(1+L)(1+|x|)|f(x)|\,.
\]  For every \(x\in D\), the function \(\theta\mapsto h(\theta,x)\) is continuous and, by assumption, the right-hand sides of both inequalities  are summable on the bounded Borel set \(D\). By continuity under the integral sign, we deduce that \(K\) is continuous on \(\Omega\times \Omega\) and that \(Q\) is continuous on \(\Omega\). We can thus apply \eqref{eq585} to \(\overline{\kappa}:=K\) and \eqref{eq565} to \(\kappa:=Q\). This gives
\begin{align*}
\lim_{\epsilon\to 0}&\left( \int_{\Omega\times \Omega}K(\theta, \vt)m_{\epsilon}(\theta)m_{\epsilon}( \vt)\dth \d \vt + \int_{\Omega}Q(\theta)m_\epsilon(\theta)\dth\right)\\&
= \int_{\Omega\times \Omega}K(\theta, \vt)\d m(\theta)\d m ( \vt) + \int_{\Omega}Q(\theta)\d m(\theta)\,.
\end{align*}
Equivalently, \(\lim_{\epsilon\to 0}\cR(f,m_\epsilon \dx)=\cR(f,m)\)
which completes the proof of~\eqref{eq372}. We have thus proved the conclusion when \(m\in \cM_c(\Omega)\). 

In the general case when \(m\in \cM_1(\Omega)\), we introduce a  sequence \((K_j)_{j\geq 1}\) of compact sets such that \(K_j\subset K_{j+1}\) and \(\cup_{j\geq 1}K_j = \Omega\). We then set \(m_j:=m\llcorner K_j\) for every \(j\geq 1\). 
Hence,
\[
\int_{\Omega}Q(\theta)\d m_j(\theta)=\int_{\Omega} \mathds{1}_{K_j}(\theta)Q(\theta) \d m (\theta)\,, 
\]
where \(\mathds{1}_{K_j}\) is the indicator function of \(K_j\) and similarly, 
\[
\int_{\Omega\times }K(\theta, \vt)\d m_j(\theta)\otimes \d m_j(\vt)=\int_{\Omega\times \Omega} \mathds{1}_{K_j\times K_j}(\theta,\vt)K(\theta,\vt) \d m (\theta)\otimes \d m(\vt)\,.
\]
Since \(m\in \cM_1(\Omega)\), it follows from \eqref{def-CKCQ} that 
\(Q\in L^{1}(\Omega, \d |m|)\) and \(K\in L^{1}(\Omega\times \Omega, \d |m|\otimes |m|)\). Hence, the dominated convergence implies that 
\[
\lim_{j\to +\infty}\int_{\Omega} Q(\theta)\d m_j(\theta) = \int_{\Omega} Q(\theta)\d m(\theta)\,,
\]
and 
\[\lim_{j\to +\infty}\int_{\Omega\times \Omega}K(\theta, \vt)\d m_j(\theta)\otimes \d m_j(\vt)=\int_{\Omega\times \Omega} (\theta,\vt)K(\theta,\vt) \d m (\theta)\otimes \d m(\vt)\,.
\]
Hence, one can deduce that
\[
\lim_{j\to +\infty}\cR(f,m_j) = \cR(f,m)\,.
\]
Together with the first part of the proof and a diagonal argument, this completes the proof.
\end{proof}

\begin{proposition}\label{prop-no-gap-M1-Mat}
For every \(m\in \cM_1(\Omega)\), there exists a sequence \((\mu_i)_{i\geq 1}\subset \cMat(\Omega)\) such that 
\[
\lim_{i\to +\infty}\cR(f,\mu_i) = \cR(f,m)\,.
\]
\end{proposition}
\begin{proof}
In view of Proposition~\ref{prop-no-gap-M1-Msm}, one can assume without loss of generality that  \(m({\rm d}x)= u(x)\dx\) for some \(u\in C^{\infty}_c(\Omega)\). 
Then, for every \(i\geq 1\), we divide \(\R^{d+1}\) into a family of closed cubes \((\Sigma^{i}_j)_{j\geq 1}\) of side \(1/i\) and with pairwise disjoint interiors. 
Observe that the diameter of each such cube is \(\sqrt{d}/i\).
We denote by \(a^{i}_{j}\) the center of the  cube \(\Sigma^{i}_j\). Let \(J_i\) the subset of those indices \(j\geq 1\) for which \(\Sigma^{i}_{j}\) intersects \(\supp u\) and let \(N_i\) be  the cardinality of \(J_i\). Finally, we let
\[
\mu_i:=\frac{1}{i^d}\sum_{j\in J_i} u(a^{i}_j)\delta_{a^{i}_j}\,.
\]
For every \(i\) larger than \(i_0:=2\sqrt{d}/\textrm{dist }(\supp u, \partial \Omega)\), one has 
\[
\textrm{dist }(\cup_{j\in J_i}\Sigma^{i}_{j}, \partial \Omega)\geq \textrm{dist }(\supp u, \partial \Omega) - \tfrac{\sqrt{d}}{i} \geq \textrm{dist }(\supp u, \partial \Omega)/2\,.
\]  
In particular, there exists a compact \(K\Subset \Omega\) that contains \(\cup_{j\in J_i}\Sigma_{j}^i\) 
%the support of \(u\) and of  every \(\mu_i\) 
for every \(i\geq i_0\). Still for \(i\geq i_0\), one has  \(\Sigma_{j}^{i}\subset \Omega\) and thus, \(|\Omega\cap \Sigma_{j}^i|=|\Sigma_{j}^i|=i^{-d}\).
Let \(\eta\in C(\Omega)\). Then,
\begin{equation}\label{eq936}
\int_{\Omega} \eta \d\mu_i = \frac{1}{i^d} \sum_{j\in J_i}(\eta u )(a^{i}_j)=\sum_{j\in J_i}(\eta u)(a^{i}_j)|\Omega\cap \Sigma_{j}^i|\,.   
\end{equation}

The function \(\eta u\) being uniformly continuous  on \(\Omega\), its modulus of continuity \(\kappa_{\eta u}(r):=\sup_{|x-y|\leq r}|(\eta u)(x)-(\eta u(y))|\) converges to \(0\) when \(r\to 0\). Moreover,
\begin{align*}
\left|\int_\Omega \eta u\dx-\sum_{j\in J_i} |\Omega\cap \Sigma^{i}_j|(\eta u)(a_{j}^i) \right|&\leq \sum_{j\in J_i} \int_{\Omega\cap \Sigma_{j}^{i}}|\eta u-(\eta u)(a^{i}_j)|\\
&\leq \kappa_{\eta u}(\sqrt{d}/i)|\cup_{j\in J_i}\Sigma^{i}_j|\leq \kappa_{\eta u}(\sqrt{d}/i)|K| \,.
\end{align*}
Hence,
\begin{equation}\label{eq946}
\lim_{i\to +\infty}\left|\int_\Omega \eta u\dx-\sum_{j\in J_i} |\Omega\cap \Sigma^{i}_j|(\eta u)(a_{j}^i) \right|=0\,.
\end{equation}
Using \eqref{eq936} and \eqref{eq946}, we deduce that
\begin{equation}\label{eq736}
\lim_{i\to +\infty}\int_{\Omega}\eta\dmu_i = \int_{\Omega}\eta u\dx\,.
\end{equation}
The above convergence result will be applied to \(\eta=Q\). To obtain a similar result for \(K\), we observe that for every \(\eta_1, \eta_2 \in C(\Omega)\), by the Fubini theorem,
\begin{multline*}
\lim_{i\to +\infty}\int_{\Omega\times \Omega} \eta_1(x)\eta_2(y)\dmu_i(x)\dmu_i(y) = \lim_{i\to +\infty}\left(\int_{\Omega}\eta_1(x)\dmu_i(x)\right)\left(\int_{\Omega}\eta_2(y)\dmu_i(y)\right)\\ = \left(\int_{\Omega}\eta_1(x)u(x)\dx\right)\left(\int_{\Omega}\eta_2(y)u(y)\dy\right)
=\int_{\Omega\times \Omega}\eta_1(x)\eta_2(y)u(x)u(y)\dx\otimes\dy\,. 
\end{multline*}
By linearity, for every \(\overline{\eta}\) of the form \((x,y)\mapsto\sum_{i\in I} \eta_{i1}(x)\eta_{i2}(y)\) where \(I\) is a finite set of indices and \(\eta_{ij}\in C(\Omega)\) for every \(i\in I\), \(j\in \{1,2\}\), one gets
\begin{equation}\label{eq724}
\lim_{i\to +\infty}\int_{\Omega\times \Omega} \overline{\eta}(x,y)\dmu_i(x)\dmu_i(y) = \int_{\Omega\times \Omega} \overline{\eta}(x,y)u(x)u(y)\dx \dy\,.  
\end{equation}
Finally, let \(\overline{\kappa}\in C(\Omega\times \Omega)\) and \(\epsilon >0\). Let \(\overline{\eta}\) as above such that \(\max_{(x,y)\in  K\times  K}|\overline{\kappa}(x,y)-\overline{\eta}(x,y)|\leq \epsilon\) (here, one relies on the Stone--Weierstrass theorem). Then,
\[
\left|\int_{\Omega\times \Omega} \overline{\kappa}(x,y)\dmu_i(x)\dmu_i(y) -
\int_{\Omega\times \Omega} \overline{\kappa}(x,y)u(x)u(y)\dx \dy
\right|
\leq I_1+I_2+I_3\,,
\]
where
\begin{align*}
I_1:&=\int_{\Omega\times \Omega} |\overline{\kappa}(x,y)-\overline{\eta}(x,y)|\d |\mu_i|(x)\d|\mu_i|(y)\,,\\
I_2:&= \left|\int_{\Omega\times \Omega} \overline{\eta}(x,y)\dmu_i(x)\dmu_i(y) -
\int_{\Omega\times \Omega} \overline{\eta}(x,y)u(x)u(y)\dx \dy
\right|\,,\\
I_3:&=\int_{\Omega\times \Omega} |\overline{\kappa}(x,y)-\overline{\eta}(x,y)||u(x)u(y)|\dx\dy\,.
\end{align*}
By construction,
\[
|\mu_i|(\Omega)\leq \frac{1}{i^d}\sum_{j\in J_i}|u(a^{i}_j)|\leq \|u\|_{L^{\infty}(\Omega)}\frac{N_i}{i^d}=\|u\|_{L^{\infty}(\Omega)}|\cup_{j\in J_i}\Sigma_{j}^i|\leq \|u\|_{L^{\infty}(\Omega)}|K|\,.
\]
Hence, \(|\mu_i|\otimes |\mu_i| (\Omega\times \Omega)\leq (\|u\|_{L^{\infty}(\Omega)}|K|)^2\). Moreover,
\[
(|u|\dx\otimes |u|\dy) (\Omega\times \Omega)=\|u\|_{L^{1}(\Omega)}^2\leq (\|u\|_{L^{\infty}(\Omega)}|K|)^2\,.
\]
It follows that $I_1, I_3\leq  (\|u\|_{L^{\infty}(\Omega)}|K|)^2\epsilon$. Together with \eqref{eq724}, this implies that
\begin{align*}
\limsup_{i\to +\infty}& \left|\int_{\Omega\times \Omega} \overline{\kappa}(x,y)\dmu_i(x)\dmu_i(y) -
\int_{\Omega\times \Omega} \overline{\kappa}(x,y)u(x)u(y)\dx \dy
\right|
\\
&\leq 2\epsilon (\|u\|_{L^{\infty}(\Omega)}|K|)^2 \\
&\qquad+  \limsup_{i\to +\infty}\left|\int_{\Omega\times \Omega} \overline{\eta}(x,y)\dmu_i(x)\dmu_i(y) -
\int_{\Omega\times \Omega} \overline{\eta}(x,y)\dmu(x)\dmu(y)
\right|\\
&=2\epsilon (\|u\|_{L^{\infty}(\Omega)}|K|)^2\,.
\end{align*}
Since this is true for every \(\epsilon>0\), we deduce therefrom that
\[
\lim_{i\to +\infty}\int_{\Omega\times \Omega} \overline{\kappa}(x,y)\dmu_i(x)\dmu_i(y)=
\int_{\Omega\times \Omega} \overline{\kappa}(x,y)u(x)u(y)\dx\dy\,.
\]
The proof is complete.
\end{proof}
{
We now have all the ingredients to present the proof of Theorem~\ref{theo:no-Lavr}.
}
\begin{proof}[Proof of Theorem~\ref{theo:no-Lavr}]

Since \(C_c^\infty(\Omega)\subset \cW \subset \cM_1(\Omega)\), we have
\[
\inf_{m \in \cM_1 (\Omega)} \cR(f,m)\leq \inf_{v \in \cW} \cR(f,v\dx) \leq \inf_{v \in C_c^\infty(\Omega)} \cR(f,v\dx)\,.
\]
By Proposition~\ref{prop-no-gap-M1-Msm}, one also has
\[
\inf_{v \in C_c^\infty(\Omega)} \cR(f,v\dx)\leq \inf_{m \in \cM_1 (\Omega)} \cR(f,m)\,.
\]
We can thus conclude that
\[
\inf_{m \in \cM_1 (\Omega)} \cR(f,m)= \inf_{v \in \cW} \cR(f,v\dx) = \inf_{v \in C_c^\infty(\Omega)} \cR(f,v\dx)\,.
\]
Similarly, the fact that \(\cMat(\Omega)\subset \cM_1 (\Omega)\) and Proposition~\ref{prop-no-gap-M1-Mat} imply that 
\[
\inf_{m \in \cM_1 (\Omega)} \cR(f,m)= \inf_{\mu\in \cMat(\Omega)} \cR(f,\mu)\,.
\]
The proof is complete.
\end{proof}

\subsection{Proof of Proposition~\ref{prop:inf-equivalence}}

Inspired by the proof of \cite[Proposition 1]{mei2018mean} written for probability measures, we provide related result for the more general  setting of finite Borel measures.
In all this section, we fix \(f\in L^{2}(D)\).
For every \(N\geq 1\), for every \(\theta_i\in \Omega\) and \(w_i\in \R\) with \(1\leq i \leq N\), the measure \(\mu_N\) defined in \eqref{eq:fN-muN} satisfies:
\begin{equation}\label{eq-R0-rhoN}
\cR(f,\mu_N)=\|f\|_{L^2(D)}^2+\frac{1}{N^2}\sum_{i,j=1}^N w_i w_j K(\theta_i, \theta_j)-\frac{2}{N}\sum_{i=1}^N w_i Q(\theta_i)\,. 
\end{equation}

\begin{remark}\label{lm-law-thetai}
Let \(M\geq 1\) and  \(\nu_+, \nu_-\) be two probability measures on \(\Omega\). Let \(\big((\overline{\theta}_{i}^+, \overline{\theta}_{i}^{-})\big)_{1\leq i \leq M}\) be a finite family of independent random variables identically distributed of law \(\nu_+\otimes \nu_-\) on \(\Omega\times \Omega\). Then, the law of each \(\overline{\theta}_{i}^+\) is \(\nu_+\), the law of each \(\overline{\theta}_{i}^{-}\) is \(\nu_-\). Moreover, all those random variables \(\overline{\theta}_{i}^\pm\) are pairwise independent.  
\end{remark}

We begin with the following observation:
\begin{lemma}
For every finite Borel measure \(m\in \cM_1(\Omega)\), 
\begin{equation}\label{prop-positive-K}
\int_{\Omega\times \Omega}K(\tau,\vt)\d m(\tau)\d m(\vt)
\geq 0\,.
\end{equation}    
\end{lemma}
\begin{proof}
By definition of \(K\), one has
\[
\int_{\Omega\times \Omega}K(\tau,\vt)\d m(\tau)\d m(\vt)
=\int_{D}\left(\int_{\Omega}h(\theta,x)\d m(\theta)\right)^2\dx\geq 0\,.    
\]
\end{proof}

Remember that  \(\cM_2(\Omega)\) denotes the set of all those finite Borel measures on \(\Omega\) such that
\[
\int_{\Omega}|\tau|^2\d |m|(\tau)<+\infty\,.
\]

\begin{lemma}\label{lemma-calculation-Montanari}
Let  \(\nu_+, \nu_-\) be two probability measures on \(\Omega\) that belong to \(\cM_2(\Omega)\) and are mutually singular. Let \((\overline{\theta}_{i}^+, \overline{\theta}_{i}^{-})_{1\leq i \leq M}\) be a finite family of independent random variables of law \(\nu_+\otimes \nu_-\) on \(\Omega\times \Omega\). For every \(\alpha_+, \alpha_-\geq 0\), we define:
\[
\overline{\rho}=\frac{\alpha_+}{N} \sum_{i=1}^{N}\delta_{\overline{\theta}_{i}^+} -\frac{\alpha_-}{N}\sum_{i=1}^{N}\delta_{\overline{\theta}_{i}^-}\,.
\]
Then, for $m=\alpha_+\nu_+ - \alpha_-\nu_-$, it holds
\[
\mathbb{E}\big(\cR(f,\overline{\rho})\big)\leq \cR(f,m)+\frac{|m|(\Omega)}{N}\int_{\Omega}K(\tau, \tau)\d |m|(\tau)\,.
\]\end{lemma}
\begin{proof}
Observe that a probability measure that belongs to \(\cM_2(\Omega)\) automatically belongs to \(\cM_1(\Omega)\). As a consequence, \(m\in \cM_1(\Omega)\) and \(\cR(f,m)\) is well-defined.
By \eqref{eq-R0-rhoN}, one has 
\begin{align*}
\mathbb{E}\big(\cR(f,\overline{\rho})\big)=& \|f\|_{L^2(D)}^2+\frac{\alpha_{+}^2}{N^2}\sum_{i,j=1}^N \mathbb{E}\big(K(\overline{\theta}_{i}^+, \overline{\theta}_{j}^+)\big)
+\frac{\alpha_{-}^2}{N^2}\sum_{i,j=1}^N \mathbb{E}\big(K(\overline{\theta}_{i}^-, \overline{\theta}_{j}^-)\big)\\
&\,-2\frac{\alpha_+\alpha_-}{N^2}\sum_{i,j=1}^N \mathbb{E}\big(K(\overline{\theta}_{i}^+, \overline{\theta}_{j}^-)\big)-\frac{2\alpha_+}{N}\sum_{i=1}^N  \mathbb{E}\big(Q(\overline{\theta}_{i}^+)\big)+\frac{2\alpha_-}{N}\sum_{i=1}^N  \mathbb{E}\big(Q(\overline{\theta}_{i}^-)\big)\,.
\end{align*}
By Remark~\ref{lm-law-thetai},  the law of \(\theta_{i}^{\pm}\) is \(\nu_{\pm}\). Hence,
\[
\mathbb{E}(Q(\overline{\theta}_{i}^\pm))=\int_{\Omega}Q(\tau)\d\nu_\pm(\tau)\,.
\]
It follows that
\begin{align*}
\frac{2\alpha_+}{N}\sum_{i=1}^N  \mathbb{E}\big(Q(\overline{\theta}_{i}^+)\big)-\frac{2\alpha_-}{N}\sum_{i=1}^N  \mathbb{E}\big(Q(\overline{\theta}_{i}^-)\big)
&= 2 \alpha_+\int_{\Omega}Q(\tau)\d\nu_+(\tau)-2
 \alpha_-\int_{\Omega}Q(\tau)\d\nu_-(\tau)\\
&= 2 \int_{\Omega}Q(\tau)\d m\,.
\end{align*}
By Remark~\ref{lm-law-thetai} again, the random variables \(\overline{\theta}_{i}^{\pm}\) are pairwise independent. It follows that
\[
\mathbb{E}\big(K(\overline{\theta}_{i}^+, \overline{\theta}_{j}^-)\big) = \int_{\Omega\times \Omega} K(\tau, \vt) \d\nu_+\otimes \d\nu_-(\tau,\vt)\,, \qquad \forall 1\leq i, j\leq N\,,
\]
and 
\[
\mathbb{E}\big(K(\overline{\theta}_{i}^\pm, \overline{\theta}_{j}^\pm)\big) = \int_{\Omega\times \Omega} K(\tau, \vt) \d\nu_\pm\otimes \d\nu_\pm(\tau,\vt)\,, \qquad \forall 1\leq i\not= j\leq N\,.
\]
Moreover,
\[
\mathbb{E}\big(K(\overline{\theta}_{i}^\pm, \overline{\theta}_{i}^\pm)\big)=\int_\Omega K(\tau, \tau)\d\nu_\pm(\tau)\,. 
\]
We thus get
\begin{align*}
\mathbb{E}\big(\cR(f,\overline{\rho})\big)&=\|f\|_{L^2(D)}^2+\frac{\alpha_{+}^2(N^2-N)}{N^2}\int_{\Omega\times \Omega}K(\tau,\vt)\d\nu_+(\tau)\d\nu_+(\vt)\\ &\quad + \frac{\alpha_{+}^2}{N}\int_{\Omega}K(\tau,\tau)\d\nu_+(\tau)
+\frac{\alpha_{-}^2(N^2-N)}{N^2}\int_{\Omega\times \Omega}K(\tau,\vt)\d\nu_-(\tau)\d\nu_-(\vt)\\
&\quad + \frac{\alpha_{-}^2}{N}\int_{\Omega}K(\tau,\tau)\d\nu_-(\tau) -2\alpha_+\alpha_-\int_{\Omega}K(\tau, \vt)\d\nu_+(\tau)\d\nu_-(\vt)
\\
&\quad -2\int_{\Omega}Q(\tau)\d m(\tau)\,.
\end{align*}
Taking into account the definition of \(m=\alpha_+\nu_+-\alpha_-\nu_-\), this gives
\begin{align*}
\mathbb{E}\big(\cR(f,\overline{\rho})\big)
&=\cR(f,m) -\frac{\alpha_{+}^2}{N}\int_{\Omega\times \Omega}K(\tau,\vt)\d \nu_+(\tau)\d \nu_+(\vt)\\
&\quad- \frac{\alpha_{-}^2}{N}\int_{\Omega\times \Omega}K(\tau,\vt)\d \nu_-(\tau)\d \nu_-(\vt)\\
&\quad +\frac{\alpha_{+}^2}{N}\int_{\Omega}K(\tau, \tau)\d \nu_+(\tau)+\frac{\alpha_{-}^2}{N}\int_{\Omega}K(\vt, \vt)\d \nu_-(\vt)\,.
\end{align*}

Hence, by \eqref{prop-positive-K},
\[
\mathbb{E}\big(\cR(f,\overline{\rho})\big)
\leq \cR(f,m)+\frac{\alpha_{+}^2}{N}\int_{\Omega}K(\tau, \tau)\d \nu_+(\tau)+\frac{\alpha_{-}^2}{N}\int_{\Omega}K(\vt, \vt)\d \nu_-(\vt)\,,
\]
which implies the desired result since \(\alpha_++\alpha_-=|m|(\Omega)\) and \(|m|=\alpha_+\nu_++\alpha_-\nu_-\).    
\end{proof}

%Since \(\cM_{2N}^{\rm at}(\Omega)\subset \cM_2(\Omega)\), one has 
%\[
%\inf_{m\in \cM_2(\Omega)}\cF(m) \leq 
%\inf_{\rho\in\cM_{2N}^{\rm at}(\Omega)}\cF(\rho)\,.
%\]

\begin{proposition}\label{prop-appendix-mat2N-m2}
For every \(N\geq 1\) and for every \(m\in \cM_2(\Omega)\),
\[
\inf_{\rho\in \cM^{\rm at}_{2N}(\Omega)} \cR(f,\rho) \leq \cR(f,m)+ \frac{|m|(\Omega)}{N} \int_{\Omega}K(\tau, \tau)\d |m|(\tau)\,.   
\]
\end{proposition}
\begin{proof}
Let \(m=m_+-m_-\) be the Haar decomposition of \(m\). We introduce the nonnegative  numbers \(\alpha_+=m_+(\Omega)\) and \(\alpha_-=m_-(\Omega)\). Let \(\nu_+, \nu_-\) be two mutually singular probability measures in \(\cM_2(\Omega)\) such that \(m_+=\alpha_+\nu_+\), \(m_-=\alpha_-\nu_-\). Let \((\overline{\theta}_{i}^+, \overline{\theta}_{i}^{-})_{1\leq i \leq M}\) be a finite family of independent random variables of law \(\nu_+\otimes \nu_-\) on \(\Omega\times \Omega\).
We then define
\[
\overline{\rho}:=\frac{\alpha_+}{N} \sum_{i=1}^{N}\delta_{\overline{\theta}_{i}^+} -\frac{\alpha_-}{N}\sum_{i=1}^{N}\delta_{\overline{\theta}_{i}^-}\,.
\]
By Lemma~\ref{lemma-calculation-Montanari}, 
\[
\mathbb{E}\big(\cR(f,\overline{\rho})\big)\leq \cR(f,m)+\frac{|m|(\Omega)}{N}\int_{\Omega}K(\tau, \tau)\d |m|(\tau)\,.
\]
There exists a choice of \(\theta_{i}^{+}, \theta_{i}^{-}\in \Omega\) with \(1\leq i \leq N\) such that the measure
\[
\rho=\frac{\alpha_+}{N} \sum_{i=1}^{N}\delta_{\theta_{i}^+} -\frac{\alpha_-}{N}\sum_{i=1}^{N}\delta_{\theta_{i}^-}
\]
satisfies 
\[
\cR(f,\rho) \leq \mathbb{E}\big(\cR(f,\overline{\rho})\big)\,.
\]
Hence, 
\[
\cR(f,\rho)\leq \cR(f,m)+\frac{|m|(\Omega)}{N}\int_{\Omega}K(\tau, \tau)\d |m|(\tau)\,. 
\]
This implies  that
\[
\inf_{\rho\in \cM_{2N}^{\rm at}(\Omega)}\cR(f,\rho) \leq \cR(f,m)+\frac{|m|(\Omega)}{N}\int_{\Omega}K(\tau, \tau)\d |m|(\tau)\,.
\]

\end{proof}

We conclude this section with the proof of Proposition~\ref{prop:inf-equivalence}. 

\begin{proof}[Proof of Proposition~\ref{prop:inf-equivalence}]
Let \(v\in L^{2}_{\omega}(\Omega)\).
Then, by the Schwarz inequality and \eqref{eq:weight-assumptions}, 
\[
\int_{\Omega} (1+|\theta|)^2 |v(\theta)|\d\theta \leq \sqrt{\co} \|v\|_{L^{2}_{\omega}(\Omega)}\,.
\]
This proves that the measure \(m:=v\dth\) belongs to \(\cM_2(\Omega)\). 
Similarly, we have
\[
|m|(\Omega)=\int_{\Omega}|v(\theta)|\d\theta \leq \sqrt{\co} \|v\|_{L^{2}_{\omega}(\Omega)}\,.
\]
Applying Proposition~\ref{prop-appendix-mat2N-m2} to this measure \(m\) and using \eqref{def-CKCQ}, one gets for every \(N\geq 2\),
\begin{align*}
\inf_{\rho\in \cM^{\rm at}_{N}(\Omega)} \cR(f,\rho)&\leq 
\inf_{\rho\in \cM^{\rm at}_{2\lfloor N/2\rfloor}(\Omega)} \cR(f,\rho)\\ &\leq \cR(f,m)+ \frac{4|m|(\Omega)}{N} \int_{\Omega}K(\tau, \tau)\d |m|(\tau)\\
&\leq \cR(f,v\d\theta)+ \frac{4\sqrt{\co}\|v\|_{L^{2}_{\omega}(\Omega)}}{N} \int_{\Omega}(C_{h}')^2(1+|\tau|)^2|v|(\tau)\d \tau\\
&\leq \cR(f,v\d\theta)+ \frac{4(C_{h}')^2\co\|v\|^{2}_{L^{2}_{\omega}(\Omega)}}{N}\,.
\end{align*}
This proves the first assertion of Proposition~\ref{prop:inf-equivalence} with \(C:=4(C_{h}')^2\co\). If one additionnally assumes that there exists a bounded minimizing sequence \((v_j)_{j\geq 1}\) for \(\cR(f,\cdot)\) in \(L^{2}_{\omega}(\Omega)\), then applying the above estimate for every \(j\geq 1\), one gets
\[
\inf_{\rho\in \cM^{\rm at}_{N}(\Omega)} \cR(f,\rho)
\leq \cR(f,v_j\d\theta)+ \frac{C}{N}\sup_{k\geq 1}\|v_k\|^{2}_{L^{2}_{\omega}(\Omega)}\,.
\]
Passing to the limit \(j\to +\infty\), one deduces that
\[
\inf_{\rho\in \cM^{\rm at}_{N}(\Omega)} \cR(f,\rho)
\leq \inf_{v\in L^{2}_\omega(\Omega)}\cR(f,v\d\theta)+ \frac{C}{N}\sup_{k\geq 1}\|v_k\|^{2}_{L^{2}_{\omega}(\Omega)}\,,
\]
which implies the desired conclusion.
\end{proof}

%In particular, if \(\cF\) has a minimizer  on \(\cM_2(\Omega)\), then
%\[
%\inf_{\rho\in \cM_{2N}^{\rm at}(\Omega)}\cF(\rho) \leq \min_{m\in \cM_{2}(\Omega)}\cF(m)+O\left(\tfrac{1}{N}\right)\,.
%\]

\subsection{Proof of Theorem~\ref{thm:convexity}}

In order to establish Theorem~\ref{thm:convexity}, and more specifically the existence and uniqueness of the minimizer of \(\cF^{(f)}_{\alpha, \beta}\),   we rely on the uniform convexity of that functional. In turn, the latter is a simple consequence of the observation already formulated in~\eqref{prop-positive-K}.

\begin{lemma}\label{lm-cRf-convex}
For every \(f\in L^{2}(D)\), the functional \(m\in \cM_1(\Omega)\to \cR(f,m)\) is convex.
\end{lemma}
\begin{proof}
By \eqref{eq:quadratic-decomposition}, the functional \(\cR(f,m)\) is the sum of an affine term \(m\mapsto \|f\|_{L^{2}(D)}^2-2\int_{\Omega}Q(\theta)\d m (\theta)\) and a   functional 
\[
m\mapsto \int_{\Omega\times \Omega}K(\theta, \vt)\d m (\theta) \d m (\vt)\,.
\]
which is quadratic (here, we use that  \(K(\theta, \vt)=K(\vt, \theta)\))  and nonnegative (as already observed in \eqref{prop-positive-K}).
It follows that \(\cR(f,\cdot)\) is convex, as desired.
\end{proof}

For later use, we introduce for every \(g\in L^{2}_{\omega}(\Omega)\) and every \(\alpha, \beta \geq 0\),  the following variant of the functional \(\cF_{\alpha, \beta}\):
\begin{multline*}
\cJ_{\alpha, \beta, g}(u):=\alpha \|u\|_{L^{2}_\omega(\Omega)}^2+\beta \|\nabla u\|_{L^{2}(\Omega)}^2 + \int_{\Omega\times \Omega}K(\theta, \vt)u(\theta)u(\vt)\dth \d\vt - \int_{\Omega}g(\theta)u(\theta)\omega(\theta)\dth\,.
\end{multline*}

\begin{lemma}\label{lm-cJ}
For every \(\alpha, \beta>0\) and every \(g\in L^{2}_\omega(\Omega)\), the functional \(\cJ_{\alpha, \beta, g}:\cW\to \R\) is \(2\min(\alpha, \beta)\)-convex on \(\cW\), and thus admits a unique minimum \(\bar{u}\) on \(\cW\). Moreover, \(\bar{u}\in W^{2,2}_{loc}(\Omega)\) and \(p(\bar{u}):=\frac{1}{\omega} \Delta \bar{u}\) satisfies that     
\[
p(\bar{u}) = \frac{1}{\beta}\left(\alpha \bar{u} + \frac{1}{\omega}\int_{\Omega}K(\cdot, \vt)\bar{u}(\vt)\d\vt-\frac{1}{2}g\right)\in L^{2}_\omega(\Omega)\,.
\]
Moreover, for every \(v\in \cW\), $\int_{\Omega} \nabla \bar{u}\cdot \nabla v =  -\int_{\Omega}p(\bar{u})v\omega$.
\end{lemma}
\begin{proof}
The map 
\[
\cJ_{0,0,g}:u\in \cW\mapsto \int_{\Omega\times \Omega}K(\theta, \vt)u(\theta)u(\vt)\dth\d\vt- \int_{\Omega}g(\theta)u(\theta)\omega(\theta)\dth
\] is convex as the sum of a convex quadratic function and an affine function. 
For every \(u\in \cW\),
\[
\cJ_{\alpha, \beta, g}(u)-\min(\alpha, \beta)\|u\|^{2}_{\cW}
=(\alpha-\min(\alpha, \beta))\|u\|^{2}_{L^{2}_\omega(\Omega)}+(\beta-\min(\alpha, \beta))\|\nabla u\|_{L^{2}(\Omega)}^2+\cJ_{0,0,g}(u)\,.
\]
Since the right-hand side is convex, we deduce that \(\cJ_{\alpha, \beta, g}\) is \(2\min(\alpha, \beta)\)-convex on \(\cW\), and thus strictly convex and coercive. 
%Since \(u\mapsto \alpha \|u\|_{L^{2}_{\omega}(\Omega)}^2+\beta \|\nabla u\|_{L^{2}(\Omega)}^2\) is strictly convex, we deduce that \(\cJ_{\alpha,\beta,g}\) is strictly convex on \(\cW\). Moreover, by \eqref{prop-positive-K} and the Schwarz inequality, 
%\begin{align*}
%\cJ_{\alpha, \beta, g}(u)&\geq \alpha \|u\|_{L^{2}_\omega(\Omega)}^2+\beta \|\nabla u\|_{L^{2}(\Omega)}^2
%-\|g\|_{L^{2}_\omega(\Omega)}\|u\|_{L^{2}_\omega(\Omega)}\\
%&\geq \tfrac{\alpha}{2} \|u\|_{L^{2}_\omega(\Omega)}^2+\beta \|\nabla u\|_{L^{2}(\Omega)}^2
%-\tfrac{1}{2\alpha}\|g\|_{L^{2}_\omega(\Omega)}\,,
%\end{align*}
% where in the last line, we have used the inequality \(ab\leq \frac{\alpha}{2}a^2 + \frac{1}{2\alpha}b^2\) for every \(a, b\geq 0\).
%We can conclude that \(\cJ_{\alpha, \beta, g}\) is coercive on \(\cW\). 
Hence, there exists a unique minimum \(\bar{u}\). Moreover, the restriction of \(\cJ_{\alpha, \beta, g}\) is smooth on \(\cW\), so that the Euler equation \(D\cJ_{\alpha, \beta, g}(\bar{u})=0\) holds, namely, for all $v\in \cW$ 
\begin{equation}\label{eq1258}
\beta \int_{\Omega}\nabla \bar{u}\cdot \nabla v\d\theta + \alpha\int_{\Omega}\bar{u} v \omega\d\theta +  \int_{\Omega\times \Omega}K(\theta, \vt)\bar{u}(\theta)v(\vt)\d\theta \d\vt -\frac{1}{2}\int_{\Omega}vg\omega\d\theta=0\,,
\end{equation}
and thus,  in the distributional sense,  
\begin{equation}\label{eq1269}
\beta\Delta\bar{u} = \alpha \omega \bar{u}+\int_{\Omega}K(\cdot,\vt)\bar{u}(\vt)\d\vt -\frac{1}{2}g\omega\,.
\end{equation}
Since \(\omega\in C^{\infty}(\Omega)\subset L^{\infty}_{loc}(\Omega)\) and \(\bar{u}, g\in L^{2}_{\omega}(\Omega)\), we deduce that \(\alpha \omega \bar{u}-\frac{1}{2}g\omega\in L^{2}_{loc}(\Omega)\). Moreover, by \eqref{def-CKCQ} and the Schwarz inequality,
\[
\left|\int_{\Omega}K(\theta, \vt)\bar{u}(\vt)\d\vt\right|\leq (C_{h}')^2\sqrt{\co}(1+|\theta|)\|\bar{u}\|_{L^{2}_\omega(\Omega)}\,.
\]
Hence,
\[
\int_{\Omega}\frac{1}{\omega(\theta)}\left( \int_{\Omega}K(\theta, \vt)\bar{u}(\vt)\d\vt\right)^2\dth
\leq (C_{h}')^4\co^2\|\bar{u}\|_{L^{2}_\omega(\Omega)}^2\,.
\]
This proves that the map 
\[
\theta\mapsto \frac{1}{\omega(\theta)}\int_{\Omega}K(\theta,\vt)\bar{u}(\vt)\d\vt
\] belongs to \(L^{2}_{\omega}(\Omega)\) and thus also to \(L^{2}_{loc}(\Omega)\) (here, we use that \(\omega\geq 1\)).
Hence, the right-hand side of \eqref{eq1269} is in \(L^{2}_{loc}(\Omega)\), so that \(\bar{u}\in W^{2,2}_{loc}(\Omega)\), see \cite[Theorem 8.8]{GT}. 
Setting \(p(\bar{u}):=\frac{\Delta\bar{u}}{\omega}\), we thus have
\begin{equation}\label{eq1275}
p(\bar{u})=\frac{1}{\beta}\left(\alpha \bar{u}+\frac{1}{\omega}\int_{\Omega}K(\cdot,\vt)\bar{u}(\vt)\d\vt -\frac{1}{2}g \right)\in L^{2}_{\omega}(\Omega)\,.
\end{equation}
%Hence, \(p(\bar{u})\in L^{2}_{\omega}(\Omega)\).

Finally, for every \(v\in \cW\), we deduce from \eqref{eq1258} that
\[
\int_{\Omega}\nabla \bar{u}\cdot \nabla v\d\theta
=\frac{-1}{\beta}\left(\alpha\int_{\Omega}\bar{u} v \omega\d\theta +  \int_{\Omega\times \Omega}K(\theta, \vt)\bar{u}(\theta)v(\vt)\d\theta \d\vt -\frac{1}{2}\int_{\Omega}vg\omega\d\theta\right)\,.
\]
Hence, by \eqref{eq1275}, one gets
\[
\int_{\Omega}\nabla \bar{u}\cdot \nabla v\d\theta
=-\int_{\Omega}p(\bar{u})v\omega\d\theta\,.
\]
The proof is complete.
\end{proof}

The regularity of the minimizer \(u^*\) stated in Theorem~\ref{thm:convexity} relies on standard elliptic estimates satisfied by the Euler equation associated to the functional \(\cF^{(f)}_{\alpha, \beta}\).  In order to fully exploit this elliptic structure, we need to establish some regularity properties of the right-hand side of the Euler equation.

\begin{lemma}\label{lm-reg}
Assume that \(h\)  satisfies \eqref{eq:h-ass} and \eqref{eq-Phi-loc-Lip}.
Given \(u^*\in L^{2}_{\omega}(\Omega)\), we consider the function
\[
\ell:\theta \in \Omega \mapsto -Q(\theta)+\int_{\Omega}K(\theta,\vt)u^*(\vt)\,d\vt.
\] 
Then, \(\ell\) is locally Lipschitz on \(\Omega\).
\end{lemma}
\begin{proof}
%From \eqref{def-CKCQ}, one gets
%\[
%\left|\int_{\Omega}K(\theta, \vt)u^*(\vt)\,d\vt \right|\leq  (C_{h}')^2 |\theta|\int_{\Omega}|\vt||u^*(\vt)|\,d\vt\leq (C_{h}')^2\hco |\theta|\|u\|_{L^{2}_\omega(\Omega)}\,.
%\]
%We deduce that \(\ell\in L^{\infty}_{loc}(\Omega)\). 

Let \(R>0\) and \(\theta, \vt \in \Omega \cap B_R\). Then, by definition of \(Q\) and \eqref{eq-Phi-loc-Lip}, 
\[
|Q(\theta)-Q(\vt)|\leq  \int_{D}|f(x)||h(\theta,x)-h(\vt,x)|\,dx \leq  c_{h,R} |\theta-\vt| \int_{D}|f(x)|(1+|x|)\,dx.
\]
By definition of \(K\) and \eqref{eq-Phi-loc-Lip} again, 
for every \(\tau\in \Omega\),
\[
|K(\theta,\tau)-K(\vt,\tau)|\leq c_{h,R}|\theta-\vt|\int_{D}(1+|x|)|h(\tau,x)|\,dx.
\]
Hence, integrating over \(\tau\in \Omega\) and using \eqref{eq:h-ass}, 
\begin{align*}
\left|\int_{\Omega}(K(\theta,\tau)-K(\vt,\tau))u^*(\tau)\d\tau\right|
&\leq c_{h,R}|\theta-\vt|\int_{D\times \Omega}(1+|x|)|h(\tau,x)||u^*(\tau)|\dx\d\tau\\
&\leq C_{h}c_{h,R}|\theta-\vt|\int_{D}(1+|x|)^2\,dx\int_{\Omega}(1+|\tau|)|u^*(\tau)|\,d\tau\\
&\leq C_{h}c_{h,R}|\theta-\vt|\int_{D}(1+|x|)^2\,dx\|u^*\|_{L^{2}_\omega(\Omega)}\sqrt{\co}\,,
\end{align*}
where the last inequality relies on \eqref{eq:h-ass} and the Schwarz inequality.
This implies that \(\ell\in W^{1,\infty}_{loc}(\Omega)\) and completes the proof.
\end{proof}

The existence and regularity of the minimizer in Theorem~\ref{thm:convexity} can now be obtained through classical  tools from convexity theory and bootstrap arguments often used in elliptic regularity. 

\begin{proof}[Proof of Theorem~\ref{thm:convexity}]
We first consider the restriction of \(\cF_{\alpha, \beta}=\cF^{(f)}_{\alpha,\beta}\) to its domain \(\cW\). From Lemma~\ref{lm-cJ} with \(g=2Q/\omega\), one deduces that   
 \(\cF_{\alpha, \beta}\) is \(2\min(\alpha, \beta)\)-convex on the Hilbert space \(\cW\), and thus attains a unique minimum \(u^*=u^*_f\), which  is a weak solution of
the linear elliptic equation:
\begin{equation}\label{eq-Euler-proof-weak}
-\beta\Delta u^* +\alpha \omega u^* +\int_{\Omega}K(\cdot, \vt)u^*(\vt)\d\vt-Q=0\,.
\end{equation}
By Lemma~\ref{lm-reg}, the function \(\ell:=-Q+\int_{\Omega}K(\cdot, \vt)u^*(\vt)\d\vt\) is locally Lipschitz on \(\Omega\).

 Using  that \(u_*\in W^{1,2}(\Omega)\) and \(\omega\in C^{\infty}(\Omega)\), we get that \(\omega u^*\in W^{1,2}_{loc}(\Omega)\).
 Since \(\ell\in W^{1,2}_{loc}(\Omega)\), classical elliptic estimates yield \(u^*\in W^{3,2}_{loc}(\Omega)\), see \cite[Theorem 8.10]{GT}. By the Sobolev embeddings, one deduces that \(u^*\in W^{1,2^{**}}_{loc}(\Omega)\), where
\[
\frac{1}{2^{**}}=\frac{1}{2^*}-\frac{1}{d+1}=\frac{1}{2}-\frac{2}{d+1}
\]
if \(d\geq 4\), while \(2^{**}\) is any number \(>1\) otherwise.  Using that \(\ell\in W^{1,2^{**}}_{loc}(\Omega)\), it follows from \eqref{eq-Euler-proof-weak} again  that \(u^*\in W^{3,2^{**}}_{loc}(\Omega)\), see \cite[Theorem 9.19]{GT}. By a standard bootstrap strategy, we can conclude that \(u^*\in W^{3,p}_{loc}(\Omega)\) for every \(p>1\). By the Morrey embeddings, this implies that \(u^*\in C^{2,s}(\Omega)\) for every \(s\in (0,1)\).

Finally, since the functional
\[
v \mapsto  \cF_{\alpha, \beta}(v) -\alpha \int_{\Omega}  v^2(\theta)\omega(\theta)\dth
\]
is convex on \(\cW\) (being a convex subset of \(L^{2}_\omega(\Omega)\)), the functional \(\cF_{\alpha, \beta}\)  is \(2\alpha\)-convex on \(L^{2}_{\omega}(\Omega)\). 

\end{proof}

We proceed with the stability results, first in \(\cW\) (Proposition~\ref{prop:stability}) and then in \(C^{2}_{loc}(\Omega)\) (Proposition~\ref{prop:stabillity_L^infty}). We also justify the continuity of \(\inf_{L^{2}_{\omega}(\Omega)}\cF^{(f)}_{\alpha, \beta}\) with respect to \(f\) (Corollary~\ref{cor-proposition-stability}). 

\begin{proof}[Proof of Proposition~\ref{prop:stability}]
Let \(\alpha, \beta>0\) and \(f\in L^{2}(D)\).
Since the corresponding functional \(\cF_{\alpha, \beta}\) has a unique minimizer \(u_f\), the latter is the unique solution of the linear equation:
\[
\alpha \int_{\Omega}u_fv\omega + \beta \int_{\Omega} \nabla u_f\cdot \nabla v + \int_{\Omega\times \Omega}K(\theta, \vt)u_f(\theta)v(\vt)\dth\d\vt =\int_{\Omega}Qv\dth\quad \forall v\in \cW\,.
\]
We deduce that \(u_f\) depends linearly on \(Q\). Since \(Q\) depends linearly on \(f\), we can conclude that \(u_f\) depends linearly on \(f\).  Inserting \(v=u_f\) in the above identity and using \eqref{prop-positive-K}, one gets
\[
\alpha \|u_f\|_{L^{2}_{\omega}(\Omega)}^2 +\beta \|\nabla u_f\|_{L^{2}(\Omega)}^2\leq \int_{\Omega}Qu_f\dth\,.
\]
Using the Schwarz and then the Young inequality in the right-hand side, one gets
\begin{equation}\label{eq1267}
\frac{\alpha}{2} \|u_f\|_{L^{2}_{\omega}(\Omega)}^2 +\beta \|\nabla u_f\|_{L^{2}(\Omega)}^2\leq \frac{1}{2\alpha}\int_{\Omega}\frac{Q(\theta)^2}{\omega(\theta)}\dth\,.
\end{equation}
From \eqref{def-CKCQ}, one gets
\[
|Q(\theta)| \leq C_h(1+|\theta|)\int_{D}(1+|x|)|f(x)|\dx \leq C' (1+|\theta|)\|f\|_{L^{2}(D)}\,,
\]
where \(C'=C'(D,h)>0\). Hence,
\begin{equation}\label{eq1284}
\int_{\Omega}\frac{Q(\theta)^2}{\omega(\theta)} \leq C'^2 \|f\|_{L^{2}(D)}^2 \int_{\Omega}\frac{(1+|\theta|)^2}{\omega(\theta)}\dth \leq  C'^2 \co\|f\|_{L^{2}(D)}^2\,, 
\end{equation}
where \(\co\) is defined in \eqref{eq:weight-assumptions}. Inserting the above estimate into \eqref{eq1267} yields the desired result.
\end{proof}

\begin{proof}[Proof of Corollary~\ref{cor-proposition-stability}]
By minimality of \(u_{f_1}^{*}\),
\[
\cF_{\alpha, \beta}^{(f_1)}(u_{f_1}^{*}) \leq \cF_{\alpha, \beta}^{(f_1)}(u_{f_2}^{*})=\cF_{\alpha, \beta}^{(f_2)}(u_{f_2}^{*})+2\int_{\Omega}(Q_{f_2}-Q_{f_1})u_{f_2}^{*}\dth\,,
\]
where \(Q_{f_i}=\int_{D}f_i(x) h(\theta,x)\dx\).
Hence, by the fact that \(Q_{f_2}-Q_{f_1}=Q_{f_2-f_1}\) and the Schwarz inequality,  one gets
\[
\cF_{\alpha, \beta}^{(f_1)}(u_{f_1}^{*})-\cF_{\alpha, \beta}^{(f_2)}(u_{f_2}^{*})\leq 2\|u_{f_2}^{*}\|_{L^{2}_\omega(\Omega)}\left(\int_{\Omega}\frac{|Q_{f_2-f_1}|^2}{\omega}\dth\right)^{1/2}\,.
\]
Relying on \eqref{eq1284} and Proposition~\ref{prop:stability}, we obtain
\[
\cF_{\alpha, \beta}^{(f_1)}(u_{f_1}^{*})-\cF_{\alpha, \beta}^{(f_2)}(u_{f_2}^{*})\leq \frac{C}{\alpha} \|f_2\|_{L^{2}(D)}\|f_2-f_1\|_{L^{2}(D)}\,,
\]
where \(C=C(\Omega, D, \omega,h)>0\). Symetrically, one also has
\[
\cF_{\alpha, \beta}^{(f_2)}(u_{f_2}^{*})-\cF_{\alpha, \beta}^{(f_1)}(u_{f_1}^{*})\leq \frac{C}{\alpha} \|f_1\|_{L^{2}(D)}\|f_2-f_1\|_{L^{2}(D)}\,.
\]
The two inequalities above yield the desired conclusion.
\end{proof}

\begin{proof}[Proof of Proposition~\ref{prop:stabillity_L^infty}]
%We already know from Proposition~\ref{prop:stability} that the correspondance \(f\in L^{2}(D)\mapsto u_f\in \cW\) is a bounded linear map.
Fix \(\Omega'\Subset \Omega\). Then, by Theorem~\ref{thm:convexity}, the restriction \(u_{f}|_{\overline{\Omega'}}\) belongs to \(C^{2}(\overline{\Omega'})\). We only need to establish the continuity of the linear map 
\begin{equation}\label{eq1288}
  f\in L^{2}(D)\mapsto u_{f}|_{\overline{\Omega'}}\in C^{2}(\overline{\Omega'})\,.  
\end{equation}
Let \((f_k)_{k\geq 1}\subset L^{2}(D)\) converge to \(f\in L^{2}(D)\) and assume that \((u_{f_k}|_{\overline{\Omega'}})_{k\geq 1}\) converges to some \(v\in C^{2}(\overline{\Omega'})\). By Proposition~\ref{prop:stability}, we know that \((u_{f_k})_{k\geq 1}\) converges to \(u_f\) in \(\cW\). We deduce that \(v=u_{f}|_{\overline{\Omega'}}\).  Since \(L^{2}(D)\) and \(C^{2}(\overline{\Omega'})\) are Banach spaces, one is entitled to apply the closed graph theorem and deduce that the map in \eqref{eq1288} is continuous, as desired.
\end{proof}

\subsection{\texorpdfstring{On the gradient flow in \(L^{2}_\omega(\Omega)\)}{On the gradient flow in L2omegaOmega}}

\subsubsection{The Hille--Yosida approach}

To study the gradient flow associated to the minimization of \(\cF_{\alpha, \beta}\), we first rely on the Hille--Yosida approach, see \cite[Chapter 7]{Brezis}.

Remember that \(A:D(A)\subset L^{2}_\omega(\Omega)\to L^{2}_\omega(\Omega)\) is the unbounded linear operator defined by 
\[
D(A)=\left\lbrace u\in \cW : \exists p=p(u)\in L^{2}_{\omega}(\Omega) \textrm{ such that } \forall v\in \cW, \int_{\Omega}\nabla u \cdot \nabla v = -\int_{\Omega}pv\omega\right\rbrace\,,
\]
\[
Au=2\alpha u -2\beta p(u)+\frac{2}{\omega}\int_{\Omega}K(\theta, \cdot)u(\theta)\dth\,.
\]
Observe that for every \(u\in D(A)\), the function \(p(u)\) is \(\Delta u/\omega\). Hence, \(\Delta u\in  L^{2}_{loc}(\Omega)\), so that by \cite[Theorem 8.8]{GT}, the function \(u\) is in \(W^{2,2}_{loc}(\Omega)\). 

\begin{lemma}
 If \(\Omega\) is \(C^2\) and bounded and \(\omega\in C^{\infty}(\overline{\Omega})\), then 
 \[
 D(A)=\left\lbrace u\in W^{2,2}(\Omega) : \tfrac{\partial u}{\partial \nu}|_{\partial \Omega}=0 \right\rbrace \,.
 \]
\end{lemma}
\begin{proof}
Since \(\Omega\) is bounded  and \(\omega\) is bounded from below and from above by positive constants, one has \(L^{2}_{\omega}(\Omega)=L^{2}(\Omega)\) and thus \(\cW=W^{1,2}(\Omega)\). Moreover, for every \(p\in L^{2}(\Omega)\), the function \(p\omega\) belongs to \(L^{2}(\Omega)\). Hence, the conclusion follows from standard elliptic estimates, see e.g. \cite[Theorem 9.26]{Brezis}.
\end{proof}

In the next lemma, we check that \(A\) satisfies all the required properties to apply the Hille--Yosida theorem.

\begin{lemma}\label{lm-prop-HY}
The unbounded linear operator \(A:D(A)\subset L^{2}_{\omega}(\Omega)\to L^{2}_\omega(\Omega)\) is an autoadjoint maximal monotone operator.    
\end{lemma}
\begin{proof}
For every \(u\in D(A)\),
\begin{align*}
\langle A(u),u\rangle_{L^{2}_\omega(\Omega)} &= 2\int_{\Omega}\left( \alpha u - \beta p(u)+\frac{1}{\omega}\int_{\Omega}K(\theta, \cdot)u\right) u\omega\\
&=2\alpha \int_{\Omega}u^2 \omega +2\beta \int_{\Omega}|\nabla u|^2 + 2 \int_{\Omega\times \Omega}K(\theta, \vt)u(\theta)u(\vt)\dth \d\vt\geq 0\,.
\end{align*}
In the last inequality, we have used \eqref{prop-positive-K}.
This proves that \(A\) is monotone.

In order to prove that \(A\) is maximal, let \(g\in L^{2}_\omega(\Omega)\) and consider the functional \(\cJ_{\alpha+1/2, \beta, g}\). Then, by Lemma~\ref{lm-cJ}, this functional  admits a unique minimizer \(\bar{u}\) on \(\cW\) that belongs to \(D(A)\) and satisfies
\begin{equation}\label{eq1271}
\beta p(\bar{u})=\beta \frac{\Delta \bar{u}}{\omega} = \left(\frac{1}{2}+\alpha\right)\bar{u} + \frac{1}{\omega}\int_{\Omega}K(\theta, \cdot)\bar{u}(\theta)\dth -\frac{1}{2}g \,.
\end{equation}
We thus  get \(\bar{u}+A(\bar{u})=g\). We can conclude that \(A\) is maximal. 

In order to prove that \(A\) is self-adjoint, we only need to establish that \(A\) is symmetric, see \cite[Proposition VII.6]{Brezis}. Let \(u, v\in D(A)\). Then,
\[
\int_{\Omega}p(u)v\omega=-\int_{\Omega}\nabla u \cdot \nabla v = \int_{\Omega}p(v)u\omega\,.
\]
We deduce therefrom that $\langle Au, v \rangle_{L^{2}_\omega(\Omega)}= \langle u, Av \rangle_{L^{2}_\omega(\Omega)} $. The proof is complete. 
\end{proof}

In the proof of Proposition~\ref{prop-Hille-Yosida-Brezis-Komura}, we exploit two important results related to gradient flows in Hilbert spaces: the Hille--Yosida theorem on the one hand, and the Br\'ezis--Komura theorem on the other hand. The latter is well-adapted to lower semicontinuous and \(\lambda\)-convex functionals. We have already checked that \(\cF_{\alpha, \beta}\) is \(2\alpha\)-convex on \(L^{2}_\omega(\Omega)\). We now verify that it is lower semicontinuous.

\begin{lemma}\label{lm-F-lsc}
The functional \(\cF_{\alpha,\beta}\) is lower semicontinuous on \(L^{2}_\omega(\Omega)\).    
\end{lemma}
\begin{proof}
Let \((u_j)_{j\geq 1}\) be a sequence on  \(L^{2}_\omega(\Omega)\) that converges to some \(u\in L^{2}_{\omega}(\Omega)\). We claim that 
\begin{equation}\label{eq1374}
\liminf_{j\to +\infty}\cF_{\alpha, \beta}(u_j)\geq  \cF_{\alpha, \beta}(u).   
\end{equation}
We can assume without loss of generality that \((\cF_{\alpha, \beta}(u_j))_{j\geq 1}\) converges in \(\R\) and that each \(u_j\) belongs to \(\cW\). By coercivity of \(\cF_{\alpha,\beta}\), this implies that \((u_j)_{j\geq 1}\) is bounded in \(\cW\). Hence, one can extract a subsequence (we do not relabel) that converges weakly in the Hilbert space \(\cW\) to some \(\tilde{u}\). By uniqueness of the limit in \(L^{2}_\omega(\Omega)\), one has \(u=\tilde{u}\). Since \(\cF_{\alpha, \beta}\) is convex and lower-semicontinuous on \(\cW\), it is also weakly sequentially lower semicontinuous on \(\cW\) and \eqref{eq1374} follows.
\end{proof}

Now that all their assumptions have been verified, it remains to apply the Hille--Yosida theorem and the Br\'ezis--Komura theorem to obtain Proposition~\ref{prop-Hille-Yosida-Brezis-Komura}.

\begin{proof}[Proof of Proposition~\ref{prop-Hille-Yosida-Brezis-Komura}]
Let \(u_0\in L^{2}_{\omega}(\Omega)\). 
Let \(v\) be the unique minimizer of \(\cJ_{\alpha, \beta, 2Q/\omega}\) given by Lemma~\ref{lm-cJ}. Then,  \(v\in D(A)\) and \(Av=2Q/\omega\).
From \cite[Theorem VII.6, Theorem VII.7]{Brezis} and Lemma~\ref{lm-prop-HY}, we deduce that there exists a unique \(\bar{u}\in C^{0}([0,\infty[;L^{2}_\omega(\Omega))\cap C^1((0,\infty);L^{2}_\omega(\Omega))\cap C^0((0,\infty);D(A))\) such that 
\[
\begin{cases}
\frac{d \bar{u}}{dt}+A\bar{u}=0\,, \\
\bar{u}(0)=u_0-v\,.
\end{cases}
\]
Moreover, \(\|\bar{u}(t)\|_{L^{2}_\omega(\Omega)}\leq \|u_0-v\|_{L^{2}_\omega(\Omega)}\), \(\|\frac{d\bar{u}}{dt}(t)\|_{L^{2}_\omega(\Omega)}=\|A\bar{u}(t)\|_{L^{2}_\omega(\Omega)}\leq \frac{1}{t}\|u_0-v\|_{L^{2}_\omega(\Omega)}\). 
%For every \(k, \ell \geq 1\), one also has \(\bar{u}\in C^{k}((0,+\infty);D(A^\ell))\). Finally, if \(u_0\in D(A)\), one has in fact \(\bar{u}\in C^{1}([0,\infty);L^{2}_\omega(\Omega))\cap C^{0}([0,\infty);D(A))\).
%There exists a continuous semigroup of contractions \cite[Chapter 7, Remark 5]{Brezis} in \(L^{2}_\omega(\Omega)\) denoted by \(S_A(t)\) such that \(\bar{u}(t)=S_A(t)u_0\). 

We then set \(u:=\bar{u}+v\). Then, \(u\) belongs to \(C^{0}([0,\infty[;L^{2}_\omega(\Omega))\cap C^1((0,\infty);L^{2}_\omega(\Omega))\cap C^0((0,\infty);D(A)) \) and satisfies

\begin{equation}\label{eq1307}
\begin{cases}
\frac{du}{dt}=  \frac{d\bar{u}}{dt} = -A\bar{u} =-A(u-v)=-Au+2\frac{Q}{\omega}\,,\\
u(0)=u_0\,.
\end{cases}
\end{equation}

Using the fact that the unbounded linear operator \(A\) is derived from the  functional \(\cF_{\alpha, \beta}\), which is lower semicontinuous (see Lemma~\ref{lm-F-lsc}) and \(2\alpha\)--convex on \(L^{2}_\omega(\Omega)\) (by Theorem~\ref{thm:convexity}), we can be more specific on the convergence of the gradient flow to the unique minimum \(u^*\) of \(\cF_{\alpha, \beta}\).
More specifically, by smoothness of \(\cF_{\alpha, \beta}\) when restricted to \(\cW\), the domain of the convex subdifferential \(\partial \cF_{\alpha, \beta}\) coincides with \(D(A)\) and for every \(u\in \cW\), 
\[
\nabla \cF_{\alpha, \beta}(u)=Au-2\tfrac{Q}{\omega}\,.
\]
The Br\'ezis--Komura theorem, see e.g. \cite[Theorem 11.7, Proposition 11.9]{AmBrSe1}, then states that the gradient flow given in \eqref{eq1307} has the following additional properties: there exists a continuous semigroup of contractions \((S_t)_{t\geq 0}\) such that
\[
u(t)=S_t u(0), 
\]
\begin{equation}\label{eq1502}
\forall v_1, v_2\in L^{2}_\omega(\Omega), \qquad \|S_t v_1 -S_t v_2\|_{L^{2}_\omega(\Omega)} \leq e^{-2\alpha t}\|v_1-v_2\|_{L^{2}_\omega(\Omega)}\,.
\end{equation}
Moreover, \(t\mapsto e^{2\alpha t}\|u'(t)\|_{L^{2}_\omega(\Omega)}\) is nonincreasing on \((0,\infty)\) and
\[
\cF_{\alpha, \beta}(u(t))\leq \inf_{v\in \cW}\left(\cF_{\alpha, \beta}(v)+\frac{\alpha}{e^{2\alpha t}-1} \|u(0)-v\|^{2}_{L^{2}_\omega(\Omega)}\right)\,.
\]
\end{proof}

\begin{proof}[Proof of Corollary~\ref{cor:exp-convergence}]

Let \(u_*\) be the minimizer of \(\cF_{\alpha, \beta}\). Then, Lemma~\ref{lm-cJ} with \(g=2Q/\omega\) implies that \(u_*\in D(A)\) and \(A(u_*)=2Q/\omega\). Hence, the gradient flow associated to the initial condition \(u_*\) is the constant map \(t\mapsto u_*\).
For every \(v\in L^{2}_\omega(\Omega)\), the estimate~\eqref{eq1502} applied to \(v_1=v\) and \(v_2=u_*\) implies that
\[
\|S_tv-u_*\|_{L^{2}_\omega(\Omega)}=\|S_t v - S_t u_*\|_{L^{2}_\omega(\Omega)}\leq e^{-2\alpha t}\|v-u_*\|_{L^{2}_\omega(\Omega)}\,.
\]
\end{proof}

We conclude this section by an interesting estimate on the approximation of the solution of a gradient flow by the solution of the corresponding  implicit Euler scheme.

\begin{theorem}[Theorem 12.5 in \cite{AmBrSe1}]
\label{thm:jko}
Let \(H\) be a Hilbert space, let  $\wt\cF:H\to [0,\infty]$ be convex and lower semicontinuous, and let $\varphi \in \mathrm{Dom}(\wt\cF)$.
Given $\tau>0$, we define 
\begin{equation}
\label{eq:jko-appendix}
\tilde{\varrho}^0 = \varphi,
\qquad
\tilde{\varrho}^{k\tau}
=
\arg\min_{\psi \in H}
\left\{
\wt\cF(\psi)
+ \tfrac{1}{2\tau}
\|\psi - \tilde{\varrho}^{(k-1)\tau}\|_{H}^2
\right\}\,.
\end{equation}
We also consider the piecewise constant left-continuous interpolation of \((\tilde{\varrho}^{k\tau})_{k\geq 0}\):
\begin{equation}\label{eq-piecewie-interpolation}
\tilde{\varrho}_{\tau}(t):=
\begin{cases}
\varphi & \textrm{ if } t=0\,,\\
\tilde{\varrho}^{k\tau} & \textrm{ if } k\geq 1 \textrm{ and }t\in ((k-1)\tau, k\tau]\,.
\end{cases}
\end{equation}
Then:
\begin{itemize}
\item[(i)] for every \(t\geq 0\), the family $(\tilde{\varrho}_\tau(t))_{\tau>0}$ is Cauchy as $\tau \to 0$, so that its limit \(\varrho(t)\) exists\,,
\item[(ii)] the curve $t\mapsto \varrho(t)$ is the gradient flow starting from $\varphi$\,,
\item[(iii)] $\|\tilde{\varrho}_\tau(t) - \varrho(t)\|_{H}
\le 2(\sqrt{2}+1) \sqrt{\tau \, \wt\cF(\varphi)}$ for every \(\tau>0\) and every \(t\geq 0\)\,.
\end{itemize} 
\end{theorem}

\subsection{Numerical simulations}
In this section, we provide more detailed explanations and derivations accompanying Section~\ref{sec:numerics}.
\subsubsection*{Derivation of quadratic form}
In this subsection we derive the approximating functional $\widehat{\mathscr{F}}_{\alpha, \beta}$. We denote 
$u \approx \widehat{u} = \sum_{i=1}^{M} a_i \widehat{u}_i$.
We consider a given dataset $\{(x_i, f(x_i))\}_{i=1}^{N_D}$.
\noindent We start with the approximating risk function:
\begin{align*}
\mathscr{R} (\vec f, \vec a) &= C_D\sum_{j=1}^{N_D} \Big(f(x_j) - \sum_{i=1}^{M} a_i \int_\Omega h(\theta, x_j) \,\widehat{u}_i(\theta) \d\theta\Big)^2 = C_D\lvert \vec f - U \vec a \rvert^2\\
&= C_D\left(\vert \vec f\,\rvert^2 - 2\vec f \,^{\top} U \vec a + \lvert U \vec a\rvert^2\right) \, ,
\end{align*}
where \(C_D=\mathcal{L}^d(D)/N_D\) (here, $\mathcal{L}^d(D)$ is the Lebesgue measure of \(D\)) and
\begin{equation*}
U_{ki} = \int_\Omega h(\theta, x_k) \,\widehat{u}_i(\theta) \d\theta \, .
\end{equation*}
Now we derive the weighted Lebesgue norm $  \lVert\widehat u\rVert^2_{L^2_\omega(\Omega)}$:
\begin{flalign*}
    \lVert\widehat u&\rVert^2_{L_\omega(\Omega)} = \Big\lVert \sum_{i=1}^{M} a_u \widehat{u}_i\Big\rVert^2_{L^2_\omega(\Omega)} = \sum_{i,j}^{M} a_i a_j \langle \widehat{u}_i, \widehat{u}_j \rangle_{L^2_\omega (\Omega)} = \vec{a}^{\top} V \, \vec{a} \, ,\quad V_{ik} = \langle \widehat{u}_i, \widehat{u}_j \rangle _{L^2_\omega(\Omega)}\, .
\end{flalign*}
And finally the gradient norm $\|\nabla  \widehat{u}\|_{L^2(\Omega)}^2$:
\begin{flalign*}
\|\nabla  \widehat{u}\|_{L^2(\Omega)}^2 
&= \sum_{\ell=1}^{d+1} \| \partial_\ell  u \|_{L^2(\Omega)}^2 
= \sum_{\ell=1}^{d+1} \Big\lVert  \sum_{i=1}^{M} a_i \partial_\ell \widehat{ u}_i \Big\rVert_{L^2(\Omega)}^2 = \sum_{d=1}^{\ell+1} \sum_{i,j = 1}^{M} a_i a_j \langle \partial_\ell\widehat{u}_i, \partial_\ell\widehat{u}_j \rangle_{L^2 (\Omega)} \\
& = \vec a \cdot W \vec a \, , \quad W_{ij} = \langle \nabla \widehat{u}_i, \nabla \widehat{u}_j \rangle_{L^{2}(\Omega)} \, .
\end{flalign*}

\subsubsection*{Basis functions}
We have used the following type of basis functions $\{\widehat {u}_i \}_{i=1}^{M}$:
\begin{enumerate}
\item \textit{Polynomials} on $\Omega = B_{R}^d \times (-L,L)$ (where for $d \geq 1$ the notation $B_{R}^d$ refers  to the ball of radius $R$ and center \(0\) in \(\R^d\)):
\begin{equation} \label{eq:polynomial}
\widehat{u}_i(\theta) =  \prod_{j=1}^{d+1} \theta_j^{p^{i}_j}
\end{equation}
\item \textit{Trigonometric orthonormal basis} on $\Omega = (-R,R)\times (-L,L)$ (for $d=1$):
\begin{equation} \label{eq:harmonic}
\widehat{u}_i (\theta) = c_1 c_2 \cos \Big(\tfrac{p_1^i \pi (\theta_0 + L)}{2L} \Big)\cos \Big(\tfrac{p_2^i \pi (\theta_1 + R)}{2R} \Big)
\end{equation}
with
\begin{equation*}
c_1 = \textstyle\begin{cases}
    1/\sqrt{2L} & p_1^i >0 \,,\\
    1/\sqrt{L} & p_1^i =0\,,
    \end{cases} \qquad  \qquad 
c_2 = \textstyle\begin{cases}
    1/\sqrt{2R} & p_2^i >0\,, \\
    1/\sqrt{R} & p_2^i =0\,.
    \end{cases}    
\end{equation*}
These functions are the solutions of the eigenvalue problem for the Laplacian: $\Delta \widehat{u}_i = \lambda\widehat{u}_i$ with zero Neumann boundary condition.
\end{enumerate} 

\subsection*{\texorpdfstring{Number of basis functions $M$}{Number of basis functions M}}
Given \(s\in \mathbb{N}_0\), we consider the set of polynomials of degree  not larger than $s$ on the space $\Omega \subseteq \R^{d+1}$. The exponents of the monomials involve all possible tuples $p^i=(p^i_j)_{1\leq j\leq d+1}$ such that $\sum_{j=1}^{d+1} p^i_j \leq s$.  The number  $M$ of such tuples can be computed as follows. 
  
We define an additional element of each tuple variable
$p^i_{d+2} := s - \sum_{j=1}^{d+1} p^i_j $
and thus the condition  $\sum_{j=1}^{d+1} p^i_j \leq s$ is equivalent to $p^i_{d+2} \ge 0$.
We obtain  $\sum_{j=1}^{d+2} p^i_j = s$ with all variables satisfying $ p^i_j \in \mathbb{N}_{0}$. So the number of all such possible tuples is given by the stars and bars formula:
\begin{equation*}
\textstyle \binom{s + (d+2) - 1}{(d+2) - 1} = \binom{s + d + 1}{d + 1} \,.
\end{equation*}

\subsection*{Additional examples}

\textbf{Example 2': Discontinuous target --- sign function ($d=1$)} 

We consider the same dataset as in Example 2. We approximate by polynomial functions $\{u_i\}_{i=1}^{M}$. The minimum of $\cF^{(f)}_{\alpha, \beta}$ is approximated by a polynomial \eqref{eq:polynomial} of order $s=60$, thus
$M = 1{,}891$. We have chosen  $\Omega=(-R,R)\times (-L,L)$ with $R=L=1.5$. For the regularized functional,  we have taken $\alpha = 3.2\times 10^{-6}$ and $\beta = 2 \times 10^{-5}$. See the results on Figure~\ref{fig3}. It is worth to emphasize that the smooth minimizer of the regularized functional does not reflect the outlier.

\begin{figure}[h!] 
\centering \includegraphics[width=0.5\textwidth]{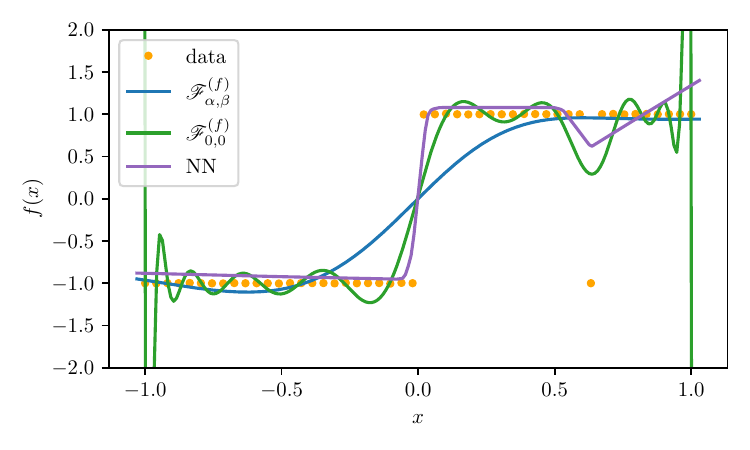} \caption{Example 2'}  
\label{fig3} 
\end{figure} 

\textbf{Example 5: Sinus of norm value for ($d=2$)} 

We generate $N_D = 100$ observations on the domain $(-1,1)^2$ of the function $f=\sin(3\lvert x\vert)$. We approximate the minimum by polynomials \eqref{eq:polynomial} of order $s=5$, so that $M=112$. We have chosen $R=L=5$ and $\alpha = 4\times 10^{-9}$ and $\beta=4\times 10^{-8}$. We note that the sparsity of the matrix $U$ is $\approx10\,\%$. See the result of approximation by the regularized functional $\widehat\cF_{\alpha,\beta}^{\ (f)}$ on Figure~\ref{fig4}. 

\begin{figure}[h!] 
\centering 
\includegraphics[width=0.5\textwidth]{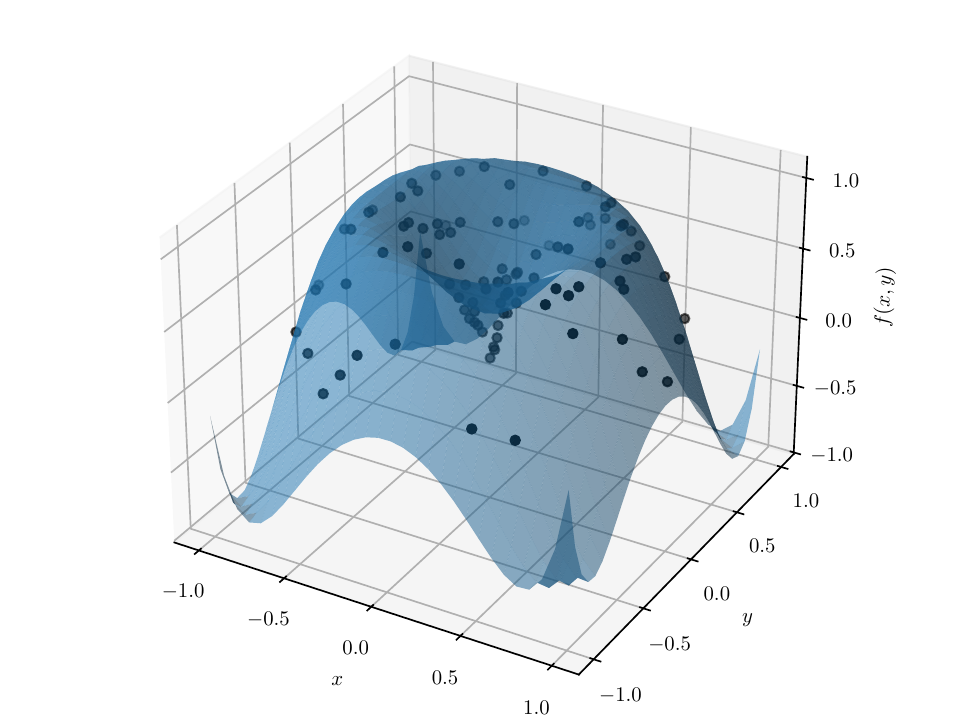} 
\caption{Example 5} 
\label{fig4} 
\end{figure}

\end{document}